\documentclass[journal]{IEEEtran} 
\usepackage[T1]{fontenc}
\usepackage[latin9]{inputenc}
\usepackage{color}
\usepackage{verbatim}
\usepackage{float}
\usepackage{amsmath}
\usepackage{amsthm}
\usepackage{amssymb}
\usepackage{stmaryrd}
\usepackage{stackrel}
\usepackage{graphicx}
\usepackage{wasysym}
\usepackage{mathtools}
\usepackage{multirow}
\usepackage[dvipsnames]{xcolor}
\usepackage{hyperref}
\makeatletter

\floatstyle{ruled}
\newfloat{algorithm}{tbp}{loa}
\providecommand{\algorithmname}{Algorithm}
\floatname{algorithm}{\protect\algorithmname}

\theoremstyle{plain}
\newtheorem{thm}{\protect\theoremname}
\theoremstyle{definition}
\newtheorem{defn}{\protect\definitionname}
\theoremstyle{definition}
\newtheorem{problem}{\protect\problemname}
\theoremstyle{plain}
\newtheorem{lem}{\protect\lemmaname}
\theoremstyle{definition}
\newtheorem{example}{\protect\examplename}
\theoremstyle{remark}
\newtheorem{rem}{\protect\remarkname}
\theoremstyle{plain}

\usepackage{cite}
\usepackage{algpseudocode}
\usepackage{setspace}
\IEEEoverridecommandlockouts

\makeatother

\usepackage{babel}
\providecommand{\corollaryname}{Corollary}
\providecommand{\definitionname}{Definition}
\providecommand{\examplename}{Example}
\providecommand{\lemmaname}{Lemma}
\providecommand{\problemname}{Problem}
\providecommand{\remarkname}{Remark}
\providecommand{\theoremname}{Theorem}

\newcommand{\UntilOp}{\mathcal{U}}
\newcommand{\Eventually}{\diamondsuit}
\newcommand{\Always}{\square}
\newcommand{\next}{\Circle}
\newcommand{\AP}{{AP}}

\SetSymbolFont{stmry}{bold}{U}{stmry}{m}{n}

\def\BibTeX{{\rm B\kern-.05em{\sc i\kern-.025em b}\kern-.08em
		T\kern-.1667em\lower.7ex\hbox{E}\kern-.125emX}}
\begin{document}

\title{Optimal Probabilistic Motion Planning with Potential Infeasible LTL
	Constraints
	
    \thanks{This work was supported in part by the National Natural Science Foundation of China under Grant 62173314, Grant U2013601, and Grant 61625303.}	
    
	\thanks{$^{1}$Department of Mechanical Engineering, Lehigh University, Bethlehem, PA, 18015, USA. $^{2}$Department of Mechanical Engineering, University of Iowa Technology Institute, The University of Iowa,
		Iowa City, IA, 52246, USA. $^{3}$Department of Automation, University of Science and Technology
		of China, Hefei, Anhui, 230026, China.}
}
		
\author{Mingyu Cai$^{1,2}$, Shaoping Xiao$^{2}$, Zhijun Li$^{3}$, and Zhen Kan$^{3}$}
\maketitle

\begin{abstract}
This paper studies optimal motion
planning
subject to motion and environment uncertainties. By modeling the system as a probabilistic labeled
Markov decision process (PL-MDP), the control objective is to synthesize a finite-memory policy, under which the agent satisfies complex high-level tasks expressed as linear temporal logic (LTL) with desired satisfaction probability.
In particular, the cost optimization of the trajectory that satisfies infinite horizon tasks is considered, and the trade-off between reducing the expected mean cost and maximizing the probability of task satisfaction is analyzed. Instead of using traditional Rabin automata, the LTL formulas are converted to limit-deterministic B\"uchi automata (LDBA) with a reachability acceptance condition and a compact graph structure. 
The novelty of this work lies in considering the cases where LTL specifications can be potentially infeasible and developing
a relaxed product MDP between PL-MDP and LDBA. The relaxed product MDP allows the agent
to revise its motion plan whenever the task is not fully feasible and quantify the revised plan's violation measurement.
A multi-objective optimization problem is then formulated to jointly
consider the probability of task satisfaction, the violation with respect to original task constraints, and the implementation cost of the policy execution. The formulated problem can be solved via coupled linear
programs. To the best of our knowledge, this work first
bridges the gap between probabilistic planning revision of potential infeasible LTL specifications and optimal control synthesis
of both plan prefix and plan suffix of the trajectory over the infinite horizons. Experimental results
are provided to demonstrate the effectiveness of the proposed framework.
	
	\global\long\def\Dist{\operatorname{Dist}}%
	\global\long\def\Inf{\operatorname{Inf}}%
	\global\long\def\Sense{\operatorname{Sense}}%
	\global\long\def\Eval{\operatorname{Eval}}%
	\global\long\def\Info{\operatorname{Info}}%
	\global\long\def\ResetRabin{\operatorname{ResetRabin}}%
	\global\long\def\Post{\operatorname{Post}}%
	\global\long\def\Acc{\operatorname{Acc}}%
\end{abstract}


\section{Introduction}

    Autonomous agents operating in complex environments
	are often subject to a variety of uncertainties. Typical uncertainties
	arise from the stochastic behaviors of the motion (e.g., potential
	sensing noise or actuation failures) and uncertain environment properties
	(e.g., mobile obstacles or time-varying areas of interest). In addition
	to motion and environment uncertainties, another layer of complexity
	in robotic motion planning is the feasibility of desired behaviors.
	For instance, areas of interest to be visited can be found to be prohibitive
	to the agent in practice (e.g., surrounded by water that the ground
	robot cannot traverse), resulting in that the user-specified tasks cannot
	be fully realized. Motivated by these challenges, this work considers
	motion planning of a mobile agent with potentially infeasible task
	specifications subject to motion and environment uncertainties, i.e., motion planning and decision making of stochastic systems. 

    Linear temporal logic (LTL) is a formal language
	capable of describing complex missions \cite{Baier2008}. For example, motion planning
	with LTL task specifications has generated substantial interest in
	robotics (cf. \cite{Kloetzer2008,Kantaros2018,srinivasan2020control}, to name a
	few). Recently, there has been growing attention in the control synthesis community to address Markov decision process (MDP) with LTL specifications based
	on probabilistic model checking, such as co-safe LTL tasks \cite{Ulusoy2014a, jagtap2020formal}, computation tree logic tasks
	\cite{Lahijanian2012}, stochastic signal temporal logic tasks \cite{Nuzzo2019},
	and reinforcement-learning-based approaches \cite{Sadigh2014,hasanbeig2018logically, hasanbeig2019certified, Modular_CAI,  cai2021safe}. 
	However, these aforementioned works only considered feasible specifications
	that can be fully executed. Thus, a challenging problem is how missions
	can be successfully managed in a dynamic and uncertain environment
	where the desired tasks are only partially feasible.

    This work studies the control synthesis of a mobile
	agent with LTL specifications that can be infeasible. The uncertainties in both robot motion (e.g., potential actuation failures) and workspace properties (e.g., obstacles or areas of interest) are
	considered. It gives rise to the probabilistic
	labeled Markov decision process (PL-MDP). Our objective is to
	generate control policies in decreasing priority order to (1) accomplish
	the pre-specified task with desired probability; (2) fulfill the pre-specified
	task as much as possible if it is not entirely feasible; and (3) minimize
	the expected implementation cost of the trajectory. Although the above
	objectives have been studied individually in the literature, this work
 	considers them together in a probabilistic manner. 

    \textbf{Related works:} From the aspect of optimization, the satisfaction of the general form of LTL tasks in stochastic systems involves the lasso-type policies comprised of a plan prefix and a plan suffix \cite{Baier2008}. When considering cost optimization subject to LTL
	specifications with infinite horizons over MDP models, the planned policies generally
	have a decision-making structure consisting of plan prefix and plan suffix. The prefix policies drive the system into an accepting maximum end component (AMEC), and the suffix policies involve the behaviors within the AMEC \cite{Baier2008}.
	Optimal policies of prefix and suffix structures have been investigated
	in the literature \cite{Svorevnova2013,Smith2011,Ding2014a,Guo2018, Forejt2011,Forejt2012}.
	A sub-optimal solution was developed in \cite{Svorevnova2013}, and
	minimizing the bottleneck cost was considered in \cite{Smith2011}.
	The works of \cite{Ding2014a,Guo2018,Forejt2011,Forejt2012} optimized
	the total cost of plan prefix and suffix while maximizing the satisfaction
	probability of specific LTL tasks. However, the aforementioned works
	\cite{Svorevnova2013,Smith2011,Ding2014a,Guo2018, Forejt2011,Forejt2012} mainly focused
	on motion planning in feasible cases and relied on a critical assumption
	of the existence of AMECs or an
	accepting run under a policy in the product MDP. Such an assumption may not be valid
	if desired tasks can not be fully completed in the operating environment. 
	
	When considering infeasible tasks, motion planning in a potential conflict situation has been partially
	investigated via control synthesis under soft LTL constraints \cite{Guo2015}
	and the minimal revision of motion plans \cite{Kim2012,Kim2012a,Tumova2013}. Recent works \cite{cai2020receding, peterson2021distributed, li2021online} extend the above approaches by considering dynamic or time-bounded temporal logic constraints. The works of \cite{vasile2017minimum} and \cite{wongpiromsarn2021minimum} leverage sampling-based methods for traffic environments.
	However, only deterministic transition systems were considered in \cite{Guo2015,Kim2012,Kim2012a,Tumova2013,cai2020receding, peterson2021distributed, li2021online,vasile2017minimum,wongpiromsarn2021minimum}. On the other hand, when considering
	probabilistic systems, a learning-based approach was utilized in the works of \cite{Cai2020} and \cite{cai2021_soft_reinforcement}. However, these works do not provide formal guarantees for multi-objective problems. The iterative temporal planning was developed in \cite{Lahijanian2016} and \cite{Lacerda2019} with partial satisfaction
	guarantees, and the work \cite{niu2020optimal} proposed a minimum violation control for finite stochastic games subject to co-safe LTL.
	These results are limited to finite horizons.
	In contrast, the satisfaction of the general LTL tasks in stochastic systems involves the lasso-type policies comprised of prefix and suffix structures \cite{Baier2008}. This work considers decision-making over infinite horizons
	in a stochastic system where desired tasks might not
	be fully feasible. In addition, this work also studies probabilistic cost optimization
	of the agent trajectory, which receives little attention in the works
	of \cite{Guo2015,Kim2012,Kim2012a,Tumova2013,cai2020receding, peterson2021distributed, li2021online,vasile2017minimum,wongpiromsarn2021minimum, Cai2020,cai2021_soft_reinforcement,Lahijanian2016}.

	  From the perspective of automaton structures, limit-deterministic B\"uchi automata (LDBA)
	are often used instead of traditional deterministic Rabin automata (DRA) \cite{Baier2008} to reduce the automaton size. It is well-known that the Rabin automata, in the
	worst case, are doubly exponential in the size of the LTL formula,
	while LDBA only has an exponential-sized automaton \cite{Sickert2016}.
	In addition, the B\"uchi accepting condition of LDBA, unlike the Rabin accepting condition,
	does not apply rejecting transitions. It allows us to constructively
	convert the  problem of satisfying the LTL formula to an almost-sure reachability
	problem \cite{Hasanbeig2019a,Bozkurt2020,cai2020reinforcement}. As a result, LDBA based control synthesis
	has been increasingly used for motion planning with LTL constraints
	\cite{Hasanbeig2019a,Bozkurt2020,cai2020reinforcement}. However, in the aforementioned
	works, cost optimization was not considered, and most of them only considered feasible cases (i.e., with goals
	to reach AMECs). In this work, the product MDP with LDBA is extended to the relaxed product
	MDP, which facilitates the optimization process to handle infeasible LTL specifications, reduces
	the automaton complexity, and improves the computational efficiency, 
    
	\textbf{Contributions:} Our work for the first time bridges the
	gap between planning revision for potentially infeasible task specifications
	and optimal control synthesis of stochastic systems subject to motion and environment uncertainties.
	In addition, we analyze the finite-memory policy of the PL-MDP that satisfies complex LTL specifications with desired probability and consider cost optimization in
	both plan prefix and plan suffix of the agent trajectory over infinite horizons. 
	The novelty of this work is the development of a relaxed product MDP between PL-MDP and LDBA to address the cases in which LTL specifications can be potentially infeasible. The relaxed product MDP allows the agent to revise its motion plan whenever the task is not fully feasible and quantify the revised plan's violation measurement. In addition, the relaxed product structure is verified to be an MDP model and a more connected directed graph. 
	Based on such a relaxed product MDP, we are able to formulate a constrained multi-objective optimization process to jointly consider the desired lower-bounded satisfaction probability of the
	task, the minimum violation cost, and the optimal implementation costs. We can find solutions by adopting coupled linear programming (LP) for MDPs relying on the network flow theory \cite{Forejt2011,Forejt2012}, which is flexible for any optimal probabilistic model checking problems. 
	We provide a comprehensive comparison with the significant existing methods, i.e., Round-Robin policy \cite{Baier2008}, PRISM \cite{Kwiatkowska2011}, and multi-objective optimization frameworks \cite{Guo2018,Forejt2011,Forejt2012}.
    Although the relaxed product MDP
	is designed to handle potentially infeasible LTL specifications, it
	is worth pointing out it is also applicable to feasible cases and
	thus generalizes most existing works. In addition, this framework can be easily adapted to formulate a hierarchical architecture that combines noisy low-level controllers and practical approaches of stochastic abstraction.
	
\section{PRELIMINARIES}

\subsection{Notations\label{subsec:notation}}

$\mathbb{N}$ represents the set of natural numbers. For an infinite path $\boldsymbol{s}=s_{0}s_{1}\ldots$ starting from state $s_0$, $\boldsymbol{s}[0]$ denotes its first element,
$\boldsymbol{s}[t], t\in\mathbb{N}$ denotes the path at step $t$, $\boldsymbol{s}[t:]$ denotes the path starting from step $t$ to the end. The expected value of a variable $x$ is $\mathbb{E}(x)$.
We use abbreviations for several notations and definitions, which are summarized in Table~\ref{tab:Abbreviation}.

\begin{table}
	\caption{\label{tab:Abbreviation} Abbreviation Summary of Notations.}
	\centering{}\resizebox{0.48\textwidth}{!}{
	\begin{tabular}{|c|c|c|c|c|c|c}
	    \hline
		\textbf{Notation Name} & \textbf{Abbreviation}  \\ 
		\hline
		\text{Limit-Deterministic B\"uchi Automaton} & \text{LDBA} \\
		\hline
		\text{Strong Connected Component} & \text{SCC} \\
		\hline
		\text{Bottom Strong Connect Component} & \text{BSCC} \\
		\hline
		\text{Accepting Bottom Strong Connect Component} & \text{ABSCC} \\
		\hline
		\text{Maximum End Component} & \text{MEC} \\
		\hline
		\text{Accepting Maximum End Component} & \text{AMEC} \\
		\hline
		\text{Average Execution Cost per Stage}  & \text{AEPS} \\
		\hline
		\text{Average Violation Cost per Stage} & \text{AVPS} \\
		\hline
		\text{Average Regulation Cost per Stage} & \text{ARPS} \\
		\hline
		\text{Linear Programming} & \text{LP} \\
		\hline
	\end{tabular}}
\end{table}

\subsection{Probabilistic Labeled MDP\label{subsec:Labeled-MDP}}

\begin{defn}
A probabilistic labeled finite MDP (PL-MDP) is a tuple $\mathcal{M}=\left(S,A,p_{S},\left(s_{0},l_{0}\right), L, p_{L}, c_{A}\right)$,
where $S$ is a finite state space, $A$ is a finite action space
(with a slight abuse of notation, $A\left(s\right)$ also denotes
the set of actions enabled at $s\in S$), $p_{S}:S\times A\times S\shortrightarrow\left[0,1\right]$
is the transition probability function, $\boldsymbol{\pi}$ is a set of atomic
propositions, and $L:S\shortrightarrow2^{\boldsymbol{\pi}}$ is a labeling function.
The pair $\left(s_{0},l_{0}\right)$ denotes an initial state $s_{0}\in S$
and an initial label $l_{0}\in L\left(s_{0}\right)$. The function
$p_{L}\left(s,l\right)$ denotes the probability of $l\subseteq L\left(s\right)$
associated with $s\in S$ satisfying $\sum_{l\in L\left(s\right)}p_{L}\left(s,l\right)=1,\forall s\in S$.
The cost function $c_{A}\left(s,a\right)$ indicates the cost of performing
$a\in A\left(s\right)$ at $s$. The transition probability $p_{S}$
captures the motion uncertainties of the agent, while the labeling
probability $p_{L}$ captures the environment uncertainties. 
\end{defn}

The PL-MDP $\mathcal{M}$ evolves by taking actions $a_{i}$ selected based on the policy at each
step $i\in\mathbb{N}_{0}$, where $\mathbb{N}_{0}=\mathbb{N}\cup\left\{ 0\right\} $.

\begin{defn}
The control policy $\boldsymbol{\mu}=\mu_{0}\mu_{1}\ldots$ is a
sequence of decision rules, which yields a path $\boldsymbol{s}=s_{0}s_{1}s_{2}\ldots$
over $\mathcal{M}$. As shown in \cite{Puterman2014}, $\boldsymbol{\mu}$
is called a stationary policy if $\boldsymbol{\mu}_{i}=\boldsymbol{\boldsymbol{\boldsymbol{\boldsymbol{\mu}}}}$ for all $i\geq0$,
where $\boldsymbol{\mu}$ can be either deterministic such that $\boldsymbol{\mu}:S\rightarrow A$
or stochastic such that $\boldsymbol{\mu}:S\times A\rightarrow\left[0,1\right]$.
The control policy $\boldsymbol{\mu}$ is memoryless if each $\mu_{i}$
only depends on its current state $s_{i}$. In contrast, $\boldsymbol{\mu}$
is called a finite-memory (i.e., history-dependent) policy if $\mu_{i}$ depends on its past
states. 
\end{defn}

In this work, we consider the stochastic policy. Let $\boldsymbol{\mu}(s)$ denote the probability distribution of actions at state $s$, and $\boldsymbol{\mu}(s, a)$ represent the probability of generating action $a$ at state $s$ using the policy $\boldsymbol{\mu}$. 

\begin{defn}
Given a PL-MDP $\mathcal{P}$ under policy $\boldsymbol{\pi}$, a  Markov
chain $MC_{\mathcal{M}}^{\boldsymbol{\mu}}$ of the PL-MDP $\mathcal{M}$ induced by a policy $\boldsymbol{\mu}$ is a tuple $\left(S,A,p^{\boldsymbol{\mu}}_{S},\left(s_{0},l_{0}\right), L, p_{L}\right)$ where $p^{\boldsymbol{\mu}}_{S}(s, s')=p_{S}(s,a,s')$ with $\boldsymbol{\mu}(s,a)>0$ for all $s,s' \in S$.
\end{defn}

\begin{defn}
A sub-MDP $\mathcal{M}{}_{\left(S',A'\right)}$ of $\mathcal{M}$ is a pair $(S', A')$ where $S'\subseteq S$ and $A'$ is a finite action space of $S'$ such that (i) $S'\neq\emptyset$, and $A'(s)\neq\emptyset , \forall s\in S'$; (ii) $\left\{ x'\in X' \mid p^{\mathcal{P}}(x,u,x')>0, \forall x\in X' \text{ and }  \forall u\in U'(x)\right\}$. An induced graph of $\mathcal{M}{}_{\left(S',A'\right)}$ is denoted as $\mathcal{G}_{\left(S',A'\right)}$ that is a directed graph, where if $p_{S}(s,a,s')>0$ with $a\in A'(s)$, for any $s,s'\in S'$, there exists an edge between $s$ and $s'$ in $\mathcal{G}_{\left(S',A'\right)}$. A sub-MDP is a strongly connected component (SCC) if its induced graph is strongly connected such that for all pairs of nodes $s,s' \in S'$, there is a path from $s$ to $s'$. A bottom strongly connected component (BSCC) is an SCC
from which no state outside is reachable by applying the restricted action space. 
\end{defn}

\begin{rem}
Note the evolution of a sub-MDP $\mathcal{M}{}_{\left(S',A'\right)}$ is restricted by the action space $A'$.
Given a PL-MDP and one of its SCCs, there may exist paths starting within the SCC and ending outside
of the SCC, whereas all paths starting from a BSCC will always stay within the same BSCC. In addition, a Markov chain $MC_{\mathcal{\mathcal{M}}}^{\boldsymbol{\pi}}$ is a sub-MDP of $\mathcal{P}$ induced by a policy $\boldsymbol{\pi}$, and its evolution is restricted by the policy $\boldsymbol{\mu}$.
\end{rem}

\begin{defn}
\cite{Baier2008} A sub-MDP $\mathcal{M}{}_{\left(S',A'\right)}$
is called an end component (EC) of $\mathcal{M}$ if it's a BSCC. An EC $\mathcal{M}{}_{\left(S',A'\right)}$ is called a maximal end
component (MEC) if there is no other EC $\mathcal{M}{}_{\left(S'',A''\right)}$
such that $S'\subseteq S''$ and $A'\left(s\right)\subseteq A''\left(s\right)$,
$\forall s\in S$.
\end{defn}

\subsection{LTL and Limit-Deterministic B\"uchi Automaton}

Linear temporal logic (LTL) is a formal language to describe the high-level specifications of a system. The ingredients of an LTL formula are a set of atomic propositions and combinations of several Boolean and temporal operators. The syntax of an LTL formula is defined inductively as
\begin{equation*}
        \phi   :=  \text{True} \mid a \mid \phi_1 \land \phi_2 \mid \lnot \phi \mid \next\phi \mid \phi_1 \UntilOp \phi_2\:, 
\end{equation*}

where $a\in\AP$ is an atomic proposition, $\text{True}$, negation $\lnot$, and conjunction $\land$ are propositional logic operators, and next $\next$ and until $\UntilOp$ are  temporal operators. The satisfaction relationship is denoted as $\models$. The semantics of an LTL formula are interpreted over words, which is an
infinite sequence $o=o_{0}o_{1}\ldots$ where $o_{i}\in2^{\AP}$ for
all $i\geq0$, and $2^{\AP}$ represents the power set of $\AP$, which are defined as:

\begin{equation*}
\arraycolsep=1.4pt
\begin{array}{lcl}
\boldsymbol{o} \models \text{true}  \\
\boldsymbol{o} \models \alpha  & \Leftrightarrow & \alpha\in  L(\boldsymbol{o}[0])  \\
\boldsymbol{o} \models \phi_{1}\land\phi_{2} &  \Leftrightarrow & \boldsymbol{o} \models \phi_{1} \text{ and } \boldsymbol{o} \models \phi_{2}  \\
\boldsymbol{o} \models \lnot\phi  & \Leftrightarrow & \boldsymbol{o} \mid\neq\phi  \\
\boldsymbol{o} \models \next\phi  & \Leftrightarrow & \boldsymbol{o}[1:] \models\phi  \\
\boldsymbol{o} \models \phi_1 \UntilOp \phi_2  & \Leftrightarrow & \exists t \text{ s.t. }\boldsymbol{o}[t:]\models\phi_{2}, \forall t'\in [0,t),  \boldsymbol{o}[t':]\models\phi_{1}  \\
\end{array} 
\end{equation*}

Alongside the standard operators introduced above, other propositional logic operators such as $\text{false}$, disjunction $\lor$, implication $\rightarrow$, and  temporal operators always $\Always$, eventually $\Eventually$ can be derived as usual. Thus an LTL formula describes a set of infinite traces through $S$.
Given an LTL formula that specifies the missions, its satisfaction can be evaluated by a limit deterministic B\"uchi
automaton (LDBA) \cite{Hahn2013,Sickert2016}.
\begin{defn}
	\label{def:LDBA} An LDBA is a tuple $\mathcal{A}=\left(Q,\Sigma\cup\left\{ \epsilon\right\},\delta,q_{0},F\right)$, 
	where $Q$ is a finite set of states, $\Sigma=2^{\boldsymbol{\AP}}$ is a finite
	alphabet, $\left\{ \epsilon\right\}$ is a set of indexed epsilons, each of which is enabled for one $\epsilon-$transition, $\delta\colon Q\times\left(\Sigma\cup\left\{ \epsilon\right\} \right)\shortrightarrow2^{Q}$
	is a transition function, $q_{0}\in Q$ is an initial state, and $F$
	is a set of accepting states. The states $Q$ can be partitioned into a deterministic set $Q_{D}$
	and a non-deterministic set $Q_{N}$, i.e., $Q=Q_{D}\cup Q_{N}$,
	where
	\begin{itemize}
		\item the state transitions in $Q_{D}$ are total and restricted within
		it, i.e., $\bigl|\delta(q,\alpha)\bigr|=1$ and $\delta(q,\alpha)\subseteq Q_{D}$
		for every state $q\in Q_{D}$ and $\alpha\in\Sigma$,
		\item the $\epsilon$-transitions  are only defined for state transitions from $Q_{N}$ to $Q_{D}$, and are not allowed in the deterministic set
		i.e., for any $q\in Q_{D}$, $\delta(q,\epsilon)=\emptyset,\forall\epsilon\in\left\{ \epsilon\right\}$,
		\item the accepting states are only in the deterministic set, i.e., $F\subseteq Q_{D}$.
	\end{itemize}	
\end{defn}
An $\epsilon-$transition allows an automaton to change its state without reading any atomic proposition. The run $\boldsymbol{q}=q_{0}q_{1}\ldots$
is accepted by the LDBA, if it satisfies the B\"uchi condition, i.e.,
$\inf\left(\boldsymbol{q}\right)\cap F\neq\emptyset$, where $\inf\left(\boldsymbol{q}\right)$
denotes the set of states that is visited infinitely often. As discussed
in \cite{Vardi1985}, the probabilistic verification of automaton
does not need to be fully deterministic. In other words, the automata-theoretic
approach still works if the restricted forms of non-determinism are allowed.
Therefore, LDBA has been applied for the qualitative and quantitative
analysis of MDPs \cite{Vardi1985,Courcoubetis1995,Hahn2013,Sickert2016}.
To convert an LTL formula to an LDBA, readers are referred to \cite{Hahn2013}. In the following analysis,
we use $\mathcal{A}_{\phi}$ to denote the LDBA corresponding to an
LTL formula $\phi$. 

\section{Problem Statement\label{sec:Problem-Statement}}

Consider an LTL task specification $\phi$ over $\boldsymbol{\pi}$ and a PL-MDP
$\mathcal{M}=\left(S,A,p_{S},\left(s_{0},l_{0}\right),\boldsymbol{\pi},L,p_{L},c_{A}\right)$.
It is assumed that the agent can sense its current state and the associated
labels. $\boldsymbol{\mu}(s_{k}, \mu_{k})$ represents the probability of selecting the control input $\mu_{k}$ at time $k$ for state $s_{k}$ using policy $\boldsymbol{\mu}$. The agent's path $\boldsymbol{s}_{\infty}^{\boldsymbol{\mu}}=s_{0}l_{0}\ldots s_{i}l_{i}s_{i+1}l_{i+1}\ldots$
under a control sequence $\boldsymbol{\mu}_{\infty}=\mu_{0}\mu_{1}\ldots$ is generated based on policy $\boldsymbol{\mu}$ such that
$s_{i+1}\in\left\{ s\in S\bigl|p_{S}\left(s_{i},\mu_{i},s\right)>0\right\} $, $\boldsymbol{\mu}(s_{i}, \mu_{i})>0$,  and $l_{i}\in L\left(s_{i}\right)$ with $p_{L}\left(s_{i},l_{i}\right)>0$.
Let $L\left(\boldsymbol{s}_{\infty}^{\boldsymbol{\mu}}\right)=l_{0}l_{1}\ldots l_{i}l_{i+1}\ldots$
be the sequence of labels associated with $\boldsymbol{s}_{\infty}^{\boldsymbol{\mu}}$, and
denote by $L\left(\boldsymbol{s}_{\infty}^{\boldsymbol{\mu}}\right)\models\phi$
	if $\boldsymbol{s}_{\infty}^{\boldsymbol{\mu}}$ satisfies $\phi$. The probability measurement of a run $\boldsymbol{s}_{\infty}^{\boldsymbol{\mu}}$ can be uniquely obtained by
	
	\begin{equation}
	    \mathbb{\Pr}_{\mathcal{M}}^{\boldsymbol{\mu}}(\boldsymbol{s}_{\infty}^{\boldsymbol{\mu}})=\stackrel[i=0]{n}{\prod}p_{L}(s_{i},l_{i})\cdot p_{S}(s_{i},\boldsymbol{\mu}_{i}(s_{i}),s_{i+1})\cdot \boldsymbol{\mu}(s_{i}, \mu_{i}).
	\end{equation}

	Then, the satisfaction probability under $\boldsymbol{\mu}$ from an initial
	state $s_{0}$ can be computed as
	\begin{equation}
	\mathbb{\Pr}_{\mathcal{M}}^{\boldsymbol{\mu}}(\phi)=\mathbb{\Pr}_{\mathcal{M}}^{\boldsymbol{\mu}}\left(\boldsymbol{s}_{\infty}^{\boldsymbol{\mu}}\in\boldsymbol{S}_{\infty}^{\boldsymbol{\mu}} \left.\right|L(\boldsymbol{s}_{\infty}^{\boldsymbol{\mu}})\models\phi\right),\label{eq:SatisProb}
	\end{equation}
	where $\boldsymbol{S}_{\infty}^{\boldsymbol{\mu}}$ is
	a set of all admissible paths under policy $\boldsymbol{\mu}$.
\begin{defn}
	\label{def:feasibility}	
	Given a PL-MDP $\mathcal{M}$, an LTL task
	$\phi$ is fully feasible if and only if $\mathbb{\Pr}_{\mathcal{M}}^{\boldsymbol{\mu}}\left(\phi\right)>0$ s.t. there exists a path $\boldsymbol{s}_{\infty}^{\boldsymbol{\mu}}$ over the infinite horizons under the policy $\boldsymbol{\mu}$ satisfying $\phi$.
\end{defn}	

Note that according to Def. \ref{def:feasibility}, an infeasible case means there exist no policies to satisfy the task, which can be interpreted as $\mathbb{\Pr}_{\mathcal{M}}^{\boldsymbol{\mu}}(\phi)=0$.

\begin{defn}
The expected average execution cost per stage (AEPS) of a PL-MDP $\mathcal{M}$ under the policy $\boldsymbol{\mu}$ is defined as
\begin{equation}
{J_{E}}(\mathcal{M}^{\boldsymbol{\mu}})=\mathbb{E}_{\mathcal{M}}^{\boldsymbol{\mu}}\left[\underset{n\rightarrow\infty}{\lim\sup}\frac{1}{n}\stackrel[i=0]{n}{\sum}c_{A}\left(s_{i},a_{i}\right)\right],\label{eq:ACPS}
\end{equation}
where $a_{i}$ is the action generated based on the policy $\boldsymbol{\mu}(s_{i})$.
\end{defn}	
	
A common objective in the literature is to find a policy $\boldsymbol{\mu}$
such that $\mathbb{\Pr}_{\mathcal{M}}^{\boldsymbol{\mu}}\left(\phi\right)$
is greater than the desired satisfaction probability while minimizing
the expected AEPS. However, when operating in a real-world environment with uncertainties in the dynamic system, the user-specified mission $\phi$
might not be fully feasible, resulting in $\mathbb{\Pr}_{\mathcal{M}}^{\boldsymbol{\mu}}\left(\phi\right)=0$
since there may not exist a path $\boldsymbol{s}_{\infty}^{\boldsymbol{\mu}}$
such that $L\left(\boldsymbol{s}_{\infty}^{\boldsymbol{\mu}}\right)\models\phi$.

\begin{example}	
	\label{ex:NoAMECs} Fig. \ref{example1} considers the properties of interests $\AP=\left\{\mathtt{Base1},\mathtt{Base2}, \mathtt{Obs}\right\}$ that label the environment and represent the regions of Base 1, Base 2, and obstacles, respectively. A robot is tasked to always eventually visit Base 1 and Base 2 while avoiding obstacles. The task can be expressed as an LTL formula $\phi_{\text{example}}=\Always\Eventually\mathtt{Base1}\land\Always\Eventually\mathtt{Base2}\land\lnot\mathtt{Obs}$.
	The labels of cells are assumed to be probabilistic, e.g., $\mathtt{Obs}:0.5$
	indicates that the likelihood of a cell occupied by an obstacle is
	0.5. 
	To model the motion uncertainty, the robot is allowed to transit between adjacent cells or stay in a cell with a set of actions $\left\{ \mathtt{Up,Right,Down,Left},\mathtt{Stay}\right\} $, and the cost of each action is equal to $2$. As shown in Fig. \ref{example1} (a), it's assumed to successfully take the desired action with
    a probability of $0.8$, and there is a probability of $0.2$ to take other perpendicular actions following a uniform distribution. There are no motion uncertainties for the action of "$\text{Stay}$".
    
	Fig. \ref{example1} (a) represents
	an infeasible case, where $\mathtt{Base}$ 1 is surrounded
	by obstacles and thus cannot be visited, while $\mathtt{Base}$
	2 is always accessible. Hence, it is desirable that the robot can revise its motion
	planning to mostly fulfill the given task (e.g., visit only $\mathtt{Base}$
	2 instead) whenever the task over an environment is found to be infeasible. Furthermore, it is essential to analyze the probabilistic violation of two different policies, as shown in Fig. \ref{example1} (b), due to the environment uncertainties. The generated trajectories have different probabilities of colliding with obstacles and result in different violation costs.
		\begin{figure}
		\centering{}\includegraphics[scale=0.28]{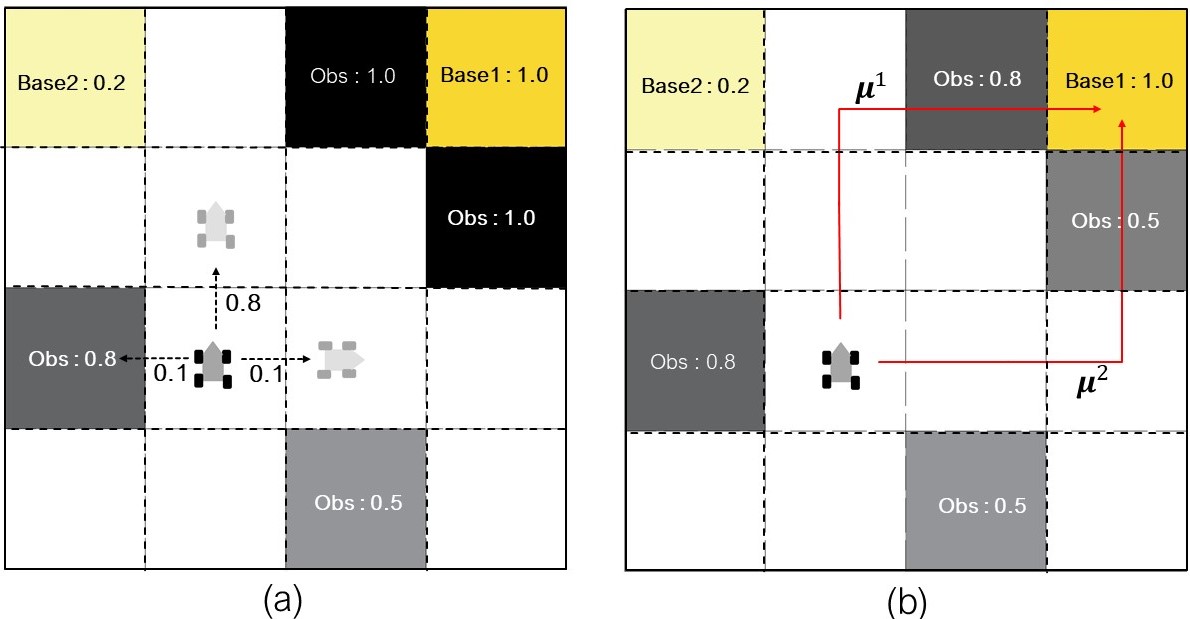}\caption{\label{example1} Example of environments where the LTL task is $\phi_{\text{example}}=\Always\Eventually\mathtt{Base1}\land\Always\Eventually\mathtt{Base2}\land\lnot\mathtt{Obs}$, and Base 1 is surrounded by obstacles with different probabilities (infeasible). (a) Motion uncertainties and inaccessible Base 1. (b) Base 1 can be visited with different risks under two different policies.}
	\end{figure}
\end{example}

As a result, to consider both feasible and infeasible tasks, a violation of task satisfaction can be defined as follows.

\begin{defn}
	\label{def:infeasibility}	
	Given a PL-MDP $\mathcal{M}$ and an LTL task $\phi$, the expected average violation cost per stage (AVPS) under the policy $\boldsymbol{\mu}$ is defined as
\begin{equation}
{J_{V}}(\mathcal{M}^{\boldsymbol{\mu}},\phi)=\mathbb{E}_{\mathcal{M}}^{\boldsymbol{\mu}}\left[\underset{n\rightarrow\infty}{\lim\inf}\frac{1}{n}\stackrel[i=0]{n}{\sum}c_{V}(s_{i}, a_{i},s_{i+1},\phi)\right],\label{eq:AVPS1}
\end{equation}
where $c_{V}(s,a,s',\phi)$ is defined as the violation cost of a transition $(s,a,s')$ with respect to $\phi$, and $a_{i}$ is the action generated based on the policy $\boldsymbol{\mu}(s_{i})$.
\end{defn}

Motivated by these challenges, the problem considered in this work
is stated as follows.

\begin{problem}
\label{Prob1}Given an LTL task $\phi$ and a PL-MDP
	$\mathcal{M}$, the goal is to find an optimal finite
	memory policy $\boldsymbol{\boldsymbol{\boldsymbol{\boldsymbol{\boldsymbol{\mu}}}}}$ from the initial state and achieve the multiple objectives with a decreasing order of priority:
	1) if $\phi$ is fully feasible, $\mathbb{\Pr}_{\mathcal{M}}^{\boldsymbol{\boldsymbol{\boldsymbol{\boldsymbol{\boldsymbol{\mu}}}}}}\left(\phi\right)\geq\gamma$, where $\gamma\in (0,1]$ is the desired satisfaction
	probability; 2) if $\phi$ is partially feasible, i.e., $\mathbb{\Pr}_{\mathcal{M}}^{\boldsymbol{\boldsymbol{\boldsymbol{\boldsymbol{\boldsymbol{\mu}}}}}}\left(\phi\right)=0$, minimizing AVPS $J_{V}(\mathcal{M}^{\boldsymbol{\mu}},\phi)$ to satisfy $\phi$ as much as possible; 3) minimizing AEPS $J_{E}(\mathcal{M}^{\boldsymbol{\mu}})$ over the infinite horizons.
\end{problem}

 Due to the consideration of infeasible
	cases, by saying to satisfy $\phi$ as much as possible in Problem \ref{Prob1}, we propose a relaxed structure and its expected average violation function to
	quantify how much the motion generated from a revised policy deviates from the desired task
	$\phi$ and minimize such a deviation. The concrete description of $c_{V}$ is introduced
	in Section \ref{subsec:Relax_PMDP}.

\section{Relaxed Product MDP Analysis}

First, Section \ref{subsec:Pro_MDP} presents the construction of LDBA-based probabilistic
product MDP. Then Section \ref{subsec:Relax_PMDP} synthesizes how it can be relaxed to handle
infeasible LTL constraints, and we concretely introduce the violation measurement of infeasible cases. Finally, the
properties of the relaxed product MDP are discussed in Section \ref{subsec:Property_relax}, which can be utilized to generate the optimal policy.

\subsection{LDBA-based Probabilistic Product MDP\label{subsec:Pro_MDP}}

We first present the definition of LDBA-based probabilistic product MDP.
\begin{defn}
	\label{def:P-MDP}Given a PL-MDP $\mathcal{M}$ and an LDBA $\ensuremath{\mathcal{A}_{\phi}}$,
	the product MDP is defined as a tuple $\mathcal{P}=\left(X,U^{\mathcal{P}},p^{\mathcal{P}},x_{0},\Acc,c_{A}^{\mathcal{P}}\right)$,
	where $X=S\times2^{\AP}\times Q$ is the set of labeled states s.t. $X=\left\{(s,l,q)\bigl|s\in S, l\in L(s), q\in Q\right\}$; $U^{\mathcal{P}}=A\cup\left\{ \epsilon\right\} $
	is the set of actions where the $\epsilon$-transitions of LDBA are regarded as actions; $x_{0}=\left(s_{0},l_{0},q_{0}\right)$ is the initial
	state; $\Acc=\left\{ \left(s,l,q\right)\in X\bigl|q\in F\right\} $
	is the set of accepting states; the cost of taking an action $u^{\mathcal{P}}\in U^{\mathcal{P}}$
	at $x=\left(s,l,q\right)$ is defined as $c_{A}^{\mathcal{P}}\left(x,u^{\mathcal{P}}\right)=c_{A}\left(s,a\right)$
	if $u^{\mathcal{P}}=a\in A\left(s\right)$ and $c_{A}^{\mathcal{P}}\left(x,u^{\mathcal{P}}\right)=0$
	otherwise;
	the transition function $p^{\mathcal{P}}:X\times U^{\mathcal{P}}\times X\shortrightarrow\left[0,1\right]$
	is defined as: for $x'=\left(s',l',q'\right)$ in $X$, 1) $p^{\mathcal{P}}\left(x,u^{\mathcal{P}},x'\right)=p_{L}\left(s',l'\right)\cdotp p_{S}\left(s,a,s^{\prime}\right)$
	if $\delta\left(q,l\right)=q^{\prime}$ and $u^{\mathcal{P}}=a\in A\left(s\right)$,
	2) $p^{\mathcal{P}}\left(x,u^{\mathcal{P}},x'\right)=1$ if $\ensuremath{u^{\mathcal{P}}\in\left\{ \epsilon\right\} }$,
	$q'\in\delta\left(q,\epsilon\right)$, and $\left(s',l'\right)=\left(s,l\right)$,
	and 3) $p^{\mathcal{P}}\left(x,u^{\mathcal{P}},x'\right)=0$ otherwise.
\end{defn}
Let $\boldsymbol{\pi}_{\mathcal{P}}$ denote the policy over $\mathcal{P}$. The product MDP  $\mathcal{P}$ captures the intersections between
all feasible paths over $\mathcal{M}$ and all words accepted to $\ensuremath{\mathcal{A}}_{\phi}$,
facilitating the identification of admissible motions that satisfy
the task $\phi$. The path $\boldsymbol{x}_{\infty}^{\boldsymbol{\pi}_{\mathcal{P}}}=x_{0}\ldots x_{i}x_{i+1}\ldots$
under a policy $\boldsymbol{\pi}_{\mathcal{P}}$
is accepted if $\inf\left(\boldsymbol{x}_{\infty}^{\boldsymbol{\pi}_{\mathcal{P}}}\right)\cap\Acc\neq\emptyset$.
If a sub-product MDP $\mathcal{P}'_{\left(X',U'\right)}$ is an MEC of $\mathcal{P}$ and $X'\cap\Acc\neq\emptyset$, $\mathcal{P}'_{\left(X',U'\right)}$
is called an accepting maximum end component (AMEC) of $\mathcal{P}$. Details of generating AMEC for a product MDP can be found in \cite{Baier2008}. Note synthesizing the AMECs doesn't require finding a set of policies that restrict the selections of actions for each state.

Denote by $\Xi_{acc}=\left\{ \Xi_{acc}^{i},i=1\ldots n_{acc}^{\mathcal{P}}\right\} $
the set of all AMECs of $\mathcal{P}$, where $\Xi_{acc}^{i}=\mathcal{P}'_{\left(X_{i}',U_{i}'\right)}$
with $X_{i}'\subseteq X$ and $U_{i}'\subseteq U^{\mathcal{P}}$ and
$n_{acc}^{\mathcal{P}}$ is the number of AMECs in $\mathcal{P}$.
Satisfying the LTL task $\phi$ is equivalent to finding a policy $\boldsymbol{\pi}_{\mathcal{P}}$ that drives the agent enter into one of an AMEC $\Xi^{i}_{acc}$ in $\mathcal{P}$. Based on that, we can define the feasibility over product MDP.

\begin{lem}
    Given a product MDP $\mathcal{P}$ constructing from a PL-MDP $\mathcal{M}$ and $\ensuremath{\mathcal{A}}_{\phi}$, the LTL task is fully feasible if and only if there exists at least one AMEC in $\mathcal{P}$ \cite{Baier2008}.
\end{lem}

As a result, if an LTL task is feasible with respect to the PL-MDP model, there exits at least one AEMC in corresponding to the product MDP, and satisfying the task $\phi$ is equivalent to reaching
an AMEC in $\Xi_{acc}$. For the cases that AMECs do not exist in $\mathcal{P}$, most existing works \cite{Baier2008, Randour(1)2015,chatterjee2011, Ding2014a}, and the work
of \cite{Guo2018} considered accepting strongly connected components
(ASCC) to minimize the probability of entering bad system states.
However, there is no guarantee that the agent will stay within an
ASCC to yield satisfactory performance, especially when the probability
of entering bad system states is large. Also, the existence
of ASCC is based on the existence of an accepting path, returns no solution for the case of Fig. \ref{example1} (b). Moreover, for the infeasible cases, the work \cite{Guo2018} needs first to check the existence of AMECs and then formulate ASCCs, whereas generating of AMECs is computationally expensive. In contrast, this frame designs a relaxed product MDP in the following, which allows us to apply its AMECs addressing both feasible and infeasible cases.

\subsection{Relaxed Probabilistic Product MDP\label{subsec:Relax_PMDP}}

For the product MDP $\mathcal{P}$ in Def. \ref{def:P-MDP},
the satisfaction of $\phi$ is based on the assumption that there
exists at least one AMEC in $\mathcal{P}$. However, such an assumption can not always be true in practice. To address this challenge, the relaxed product MDP is designed
to allow the agent to revise its motion plan whenever 
the desired LTL constraints cannot be strictly followed.
\begin{defn}
	\label{def:relaxed-product} The relaxed product MDP is constructed
	from $\mathcal{P}$ as a tuple $\mathcal{R}=\left(X,U^{\mathcal{R}},p^{\mathcal{R}},x_{0},\Acc,c_{A}^{\mathcal{R}},c_{V}^{\mathcal{R}}\right)$
	, where
	\begin{itemize}
	    \item $X$, $x_{0}$, and $\Acc$ are the same as in $\mathcal{P}$.
	    
	    \item $U^{\mathcal{R}}$ is the set of extended actions that are extended to jointly consider the actions of $\mathcal{M}$ and
	the input alphabet of $\ensuremath{\mathcal{\mathcal{A}_{\phi}}}$. Specifically, given a state $x=\left(s,l,q\right)\in X$,
	the available actions are $U^{\mathcal{R}}\left(x\right)=\left\{ \left(a,\iota\right)\bigl|a\in A\left(s\right),\iota\in\left(2^{\AP}\cup\left\{\epsilon\right\}\right)\right\} $.
	Given an action $u^{\mathcal{R}}=\left(a,\iota\right)\in U^{\mathcal{R}}\left(x\right)$,
	the projections of $u^{\mathcal{R}}$ to $A\left(s\right)$ in $\mathcal{M}$
	and to $2^{\AP}\cup\left\{ \epsilon\right\} $ in $\mathcal{A}_{\phi}$
	are denoted by $u\bigr|_{\mathcal{M}}^{\mathcal{\mathcal{R}}}$ and
	$u\bigr|_{\ensuremath{\mathcal{A}}}^{\mathcal{R}}$, respectively.
	
	$\text{ }$
	
	\item $p^{\mathcal{R}}:X\times U^{\mathcal{R}}\times X\shortrightarrow\left[0,1\right]$
	is the transition function. The transition probability $p^{\mathcal{R}}$ from a state $x=\left(s,l,q\right)$
	to a state $x'=\left(s',l',q'\right)$ is defined as: 1) $p^{\mathcal{R}}\left(x,u^{\mathcal{R}},x'\right)=p_{L}\left(s',l'\right)\cdotp p_{S}\left(s,a,s^{\prime}\right)$
	with $a=u\bigr|_{\mathcal{M}}^{\mathcal{\mathcal{R}}}$, if $q$ can
	be transited to $q^{\prime}$ and $u\bigr|_{\ensuremath{\mathcal{A}}}^{\mathcal{R}}\neq\epsilon$
	and $\ensuremath{\mathcal{\delta}}\left(q,u\bigr|_{\ensuremath{\mathcal{A}}}^{\mathcal{R}}\right)=q'$;
	2) $p^{\mathcal{R}}\left(x,u^{\mathcal{R}},x'\right)=1$, if $u\bigr|_{\ensuremath{\mathcal{A}}}^{\mathcal{R}}=\epsilon$,
	$q'\in\delta\left(q,\epsilon\right)$, and $\left(s',l'\right)=\left(s,l\right)$;
	3) $p^{\mathcal{R}}\left(x,u^{\mathcal{R}},x'\right)=0$ otherwise.
	Under an action $u^{\mathcal{R}}\in U^{\mathcal{R}}\left(x\right)$,
	it holds that $\sum_{x'\in X}p^{\mathcal{R}}\left(x,u^{\mathcal{R}},x'\right)=1$.
	
	$\text{ }$
	
	\item $c_{V}^{\mathcal{R}}:X\times U^{\mathcal{R}}\shortrightarrow\mathbb{R}$ is the execution cost. Given a state $x$ and an action $u^{\mathcal{R}}$, the execution
	cost
	is defined as
	\[
	c_{A}^{\mathcal{R}}(x,u^{\mathcal{R}})=\left\{ \begin{array}{cc}
	c_{A}\left(s,a\right) & \text{ if \ensuremath{u\bigr|_{\mathcal{M}}^{\mathcal{\mathcal{R}}}\in A\left(s\right),} }\\
	0 & \text{otherwise.}
	\end{array}\right.
	\]
	$\text{ }$
	
	\item $c_{V}^{\mathcal{R}}:X\times U^{\mathcal{R}}\times X\shortrightarrow\mathbb{R}$
	is the violation cost.
	The violation cost of the transition from $x=\left(s,l,q\right)$
	to $x'=\left(s',l',q'\right)$ under an action $u^{\mathcal{R}}$
	is defined as
	\[
	c_{V}^{\mathcal{R}}\left(x,u^{\mathcal{R}},x'\right)=\left\{ \begin{array}{cc}
	p_{L}\left(s',l'\right)\cdotp w_{V}\left(x,x'\right) & \text{ if \ensuremath{u\bigr|_{\ensuremath{\mathcal{A}}}^{\mathcal{R}}\neq\epsilon}, }\\
	0 & \text{otherwise},
	\end{array}\right.
	\]
	where $w_{V}\left(x,x'\right)=\Dist\left(L\left(s\right),\ensuremath{\mathcal{X}}\left(q,q^{\prime}\right)\right)$
	with $\ensuremath{\mathcal{X}}\left(q,q'\right)=\left\{ l\in2^{\boldsymbol{\pi}}\left|q\overset{l}{\shortrightarrow}q'\right.\right\} $
	being the set of input alphabets that enables the transition from
	$q$ to $q^{\prime}$. Borrowed from \cite{Guo2015}, the function
	$\Dist\left(L\left(s\right),\ensuremath{\mathcal{X}}\left(q,q^{\prime}\right)\right)$
	measures the distance from $L\left(s\right)$ to the set $\ensuremath{\mathcal{X}}\left(q,q^{\prime}\right)$.
	\end{itemize}
	
\end{defn}

\begin{rem}
Given a PL-MDP $\mathcal{M}$ and $A_{\phi}$, the relaxed product MDP $\mathcal{R}$ holds the same state space as the corresponding product MDP $\mathcal{P}$. The main difference compared with $\mathcal{P}$ is that the $\mathcal{R}$ has a different action space with revised transition conditions so that $\mathcal{R}$ has a more connected structure. In addition, we propose the violation cost for each transition to measure the AVPS over the relaxed product model $\mathcal{R}$. The complexity analysis of applying the relaxed product MDP is discussed in Section \ref{sec:complexity}. Note that the environment uncertainties influence the transition probabilities of a relaxed product MDP, and in turn, affect the probabilities of entering into AMECs.
\end{rem}

The weighted violation function $w_{V}\left(x,x'\right)$ quantifies
how much the transition from $x$ to $x'$ in a product automaton
violates the constraints imposed by $\phi$. It holds that $c_{V}^{\mathcal{R}}\left(x,u^{\mathcal{R}},x'\right)=0$
if $p^{\mathcal{P}}\left(x,u^{\mathcal{P}},x'\right)\neq0$, since
a non-zero $p^{\mathcal{P}}\left(x,u^{\mathcal{P}},x'\right)$ indicates
either $\delta\left(q,L\left(s\right)\right)=q'$ or $\delta\left(q,\epsilon\right)=q'$,
leading to $w_{V}\left(x,x'\right)=0$. Let $\boldsymbol{\pi}$ denote the policy of $\mathcal{R}$. 
Consequently, we can transform the measurement of AEPS, and AVPS $J_{V}(\mathcal{M}^{\boldsymbol{\mu}},\phi)$ from PL-MDP $\mathcal{M}$ into $\mathcal{R}$. 

\begin{defn}
Given a relaxed product MDP $\mathcal{R}$ generated from a PL-MDP $\mathcal{M}$ and an LDBA $A_{\phi}$, the AEPS of $\mathcal{R}$ under policy $\boldsymbol{\pi}$ can be defined as:
\begin{equation}
J_{E}(\mathcal{R}^{\boldsymbol{\pi}})=\mathbb{E}_{\mathcal{R}}^{\boldsymbol{\boldsymbol{\boldsymbol{\pi}}}}\left[\underset{n\rightarrow\infty}{\lim\inf}\frac{1}{n}\stackrel[i=0]{n}{\sum}c_{A}^{\mathcal{R}}(x_{i},u_{i}^{\mathcal{R}})\right].\label{eq: AEPS}
\end{equation}
Similarly, the AVPS of $\mathcal{R}$ can be reformulated as: 
\begin{equation}
J_{V}(\mathcal{R}^{\boldsymbol{\pi}})=\mathbb{E}_{\mathcal{R}}^{\boldsymbol{\boldsymbol{\boldsymbol{\pi}}}}\left[\underset{n\rightarrow\infty}{\lim\inf}\frac{1}{n}\stackrel[i=0]{n}{\sum}(c_{V}^{\mathcal{R}}\left(x_{i},u_{i}^{\mathcal{R}},x_{i+1}\right))\right].\label{eq: AVPS}
\end{equation}
\end{defn}

Hence, $J_{V}(\mathcal{R}^{\boldsymbol{\pi}})$
can be applied to measure how much $\phi$ is satisfied in Problem
\ref{Prob1}.
It should be pointed out that the violation cost $c_{V}^{\mathcal{R}}$
	jointly considers the probability of an event $p_{L}\left(s',l'\right)$
	and the violation of the desired $\phi$. For instance, Fig. \ref{example1} (b) shows the trajectories generated from two different policies that traverse regions labeled $\text{Obs}$ with different probabilities. It's obvious that the task of infinitely visiting $\text{Base1}$ and $\text{Base2}$ is infeasible. The paths induced from different policies hold different AVPSs for partial satisfaction. Consequently, a large cost
	$c_{V}^{\mathcal{R}}$ can occur if $p_{L}\left(s',l'\right)$
	is close to 1 (e.g., an obstacle appears with high probability), or
	the violation $w_{V}$ is large, or both are large. Hence, minimizing
	the AVPS will not only bias the planned path
	towards more fulfillment of $\phi$ by penalizing $w_{V}$, but also
	towards more satisfaction of mission operation (e.g., reduce the risk
	of mission failures by avoiding areas with high probability obstacles).
	This idea is illustrated via simulations in Case 2 in Section
	\ref{sec:Case}.

\subsection{Properties of Relaxed Product MDP\label{subsec:Property_relax}}

Given an LTL formula $\phi$ and a PL-MDP $\mathcal{M}$, this section verifies properties of the designed relaxed product
MDP $\mathcal{R}$, which can be applied to solve feasible cases where there exists at least one policy $\boldsymbol{\mu}$ such that $\mathbb{\Pr}_{\mathcal{M}}^{\boldsymbol{\mu}}\left(\phi\right)>0$, and infeasible cases where $\mathbb{\Pr}_{\mathcal{M}}^{\boldsymbol{\mu}}\left(\phi\right)=0$ for any policy $\boldsymbol{\mu}$.
Based on definition \ref{def:relaxed-product}, the relaxed product MDP $\mathcal{R}$ and its corresponding product MDP $\mathcal{P}$ have the same states. Hence, we can regard $\mathcal{R}$ and $\mathcal{P}$ as two separate directed graphs. Let ABSCC denote the BSCC that contains at least one accepting state in $\mathcal{P}$ or $\mathcal{R}$.
\begin{thm}
	\label{Thm_Properties}Given a PL-MDP $\mathcal{M}$ and an LDBA automaton
	$\mathcal{A}_{\phi}$ corresponding to the desired LTL task specification
	$\phi$, the relaxed product MDP $\mathcal{\mathcal{R}=M}\otimes\mathcal{A}_{\phi}$ and corresponding product MDP $\mathcal{P}$
	have the following properties:
	\begin{enumerate}
		\item the directed graph of traditional product $\mathcal{P}$ is sub-graph of
		 the directed graph of $\mathcal{R}$,
		\item there always exists at least one AMEC in $\mathcal{R}$,
		\item if the LTL formula $\phi$ is feasible over $\mathcal{M}$, any direct graph of AMEC of $\mathcal{P}$ is the sub-graph of a direct graph of AMEC of $\mathcal{R}$.
	\end{enumerate}
\end{thm}
\begin{proof}
	Property 1: by definition \ref{def:P-MDP}, there is a transition between
	$x=\left\langle s,l,q\right\rangle $ and $x'=\left\langle s',l',q'\right\rangle $
	in $\mathcal{P}$, if and only if $p^{\mathcal{P}}\left(x,u^{\mathcal{P}},x'\right)\neq0$.
	There are two cases for $p^{\mathcal{P}}\left(x,u^{\mathcal{P}},x'\right)\neq0$:
	i) $\exists l\in L(s), \delta(q,l)=q'$ and $p_{S}\left(s,a,s'\right)\neq0$
	with $u^{\mathcal{P}}=a$; and ii) $q'\in\delta\left(q,\epsilon\right)$
	and $u^{\mathcal{P}}=\epsilon$. In the relaxed $\mathcal{R}$, for
	case i), there always exist $u\bigr|_{\ensuremath{\mathcal{A}}}^{\mathcal{R}}=L\left(s\right)$
	and $u\bigr|_{\ensuremath{\mathcal{\mathcal{M}}}}^{\mathcal{\mathcal{R}}}=u^{\mathcal{P}}=a$
	with $p_{S}\left(s,a,s'\right)\neq0$ such that $p^{\mathcal{R}}\left(x,u^{\mathcal{R}},x'\right)\neq0$.
	For case ii), based on the fact that $q'\in\delta\left(q,\epsilon\right)$,
	there always exists $u\bigr|_{\ensuremath{\mathcal{A}}}^{\mathcal{R}}=\epsilon$
	such that $p^{\mathcal{R}}\left(x,u^{\mathcal{R}},x'\right)\neq0$.
	Therefore, any existing transition in $\mathcal{P}$ is also
	preserved in the corresponding relaxed product MDP $\mathcal{R}$. 
	
	Property 2: as indicated in \cite{Sickert2016}, for an LDBA $\mathcal{A}_{\phi}$,
	there always exists a BSCC that contains at least one of the accepting
	states. Without loss of generality, let $Q_{B}\subseteq Q$ be a BSCC
	of $\mathcal{A}_{\phi}$ s.t.  $Q_{B}\cap F\neq\emptyset$. Denote
	by $\mathcal{M}_{\left(S',A'\right)}$ an EC of $\mathcal{M}$.
	By the definition of the relaxed product MDP $\mathcal{\mathcal{R}=M}\otimes\mathcal{A}_{\phi}$,
	we can construct a sub-product MDP $\mathcal{\mathcal{R}}_{\left(X_{B},U_{B}^{\mathcal{R}}\right)}$
	such that $x=\left\langle s,l,q\right\rangle \in X_{B}$ with $s\in S'$
	and $q\in Q_{B}$. For each $u_{B}^{\mathcal{R}}\left(x\right)\in U_{B}^{\mathcal{R}}$,
	we restrict $u_{B}^{\mathcal{R}}\left(x\right)=\left(A\left(s\right),l_{B}\right)$
	with $A\left(s\right)\in A'$ and $\delta\left(q,l_{B}\right)\in Q_{B}$.
	As a result, we can obtain that an EC $\mathcal{\mathcal{R}}_{\left(X_{B},U_{B}^{\mathcal{R}}\right)}$
	that contains at least one of the accepting states due to the fact i.e. $Q_{B}\cap F\neq\emptyset$. Therefore, there exists at
	least an AMEC in the relaxed $\mathcal{R}$.
	
	Property 3: if $\phi$ is feasible over $\mathcal{M}$, there exist
	AMECs in both $\mathcal{P}$ and $\mathcal{R}$. Let $\Xi_{\mathcal{P}}$
	and $\Xi_{\mathcal{\mathcal{R}}}$ be an AMEC of $\mathcal{P}$ and
	$\mathcal{R}$, respectively. From graph perspectives, $\Xi_{\mathcal{P}}$
	and $\Xi_{\mathcal{\mathcal{R}}}$ can be considered as BSCCs $\mathcal{G}_{\left(\Xi_{\mathcal{P}}\right)}\subseteq\mathcal{G}_{\left(X,U^{\mathcal{\mathcal{P}}}\right)}$
	and $\mathcal{G}_{\left(\Xi_{\mathcal{\mathcal{R}}}\right)}\subseteq\mathcal{G}_{\left(X,U^{\mathcal{R}}\right)}$
	containing accepting states, respectively. According to Property 1,
	it can be concluded that for any $\mathcal{G}_{\left(\Xi_{\mathcal{P}}\right)}$, we can find a $\mathcal{G}_{\left(\Xi_{\mathcal{R}}\right)}$ s.t. $\mathcal{G}_{\left(\Xi_{\mathcal{P}}\right)}$ is a sub-graph of $\mathcal{G}_{\left(\Xi_{\mathcal{R}}\right)}$
\end{proof}

\begin{figure}
	\centering{}\includegraphics[scale=0.60]{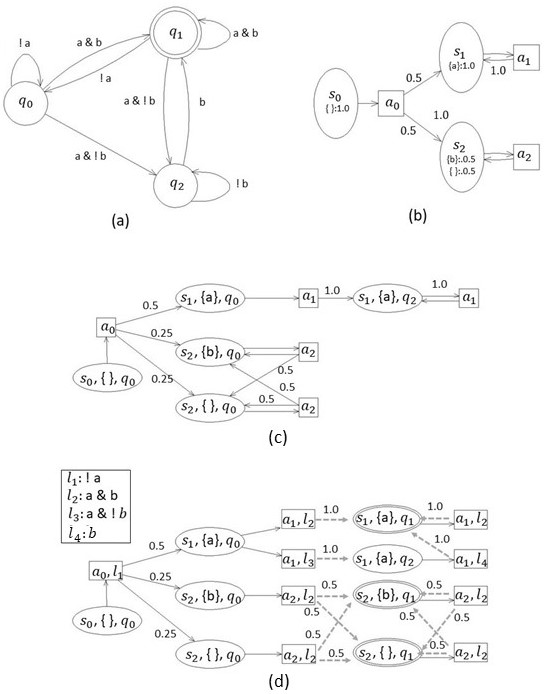}\caption{\label{fig:LDBA =000026 MDP =000026 product} (a) The LDBA $\mathcal{A}_{\phi}$.
		(b) The MDP $\mathcal{M}$. (c) The constrained product MDP. (d) The
		relaxed product MDP.}
\end{figure}

Theorem \ref{Thm_Properties} indicates that the directed graph of $\mathcal{R}$
is more connected than the directed graph of the corresponding $\mathcal{P}$. Therefore, there always exists at least one AMEC in $\mathcal{R}$ even for the infeasible cases, which allows us to measure the violation with respect to the original LTL formula. Moreover,
if a given task $\phi$ is fully feasible in $\mathcal{P}$ (
there exists a policy $\boldsymbol{\pi}_{\mathcal{P}}$ such that its induced path $\boldsymbol{x}_{\boldsymbol{\pi}^{\mathcal{P}}}$
over $\mathcal{P}$ satisfying $\phi$ i.e. $\mathbb{\Pr}_{\mathcal{M}}^{\boldsymbol{\mu}}\left(\phi\right)>0$). Also, there must exist a policy $\boldsymbol{\pi}$. s.t. its induced path $\boldsymbol{x}^{\boldsymbol{\pi}}$ over $\mathcal{R}$ is free of
violation cost. In other words, $\mathcal{R}$ can also
handle feasible tasks by identifying accepting paths with zero AVPS.

\begin{example}
	To illustrate Theorem \ref{Thm_Properties}, a running
	example is shown here. Consider an LDBA $\mathcal{A}_{\phi}$ corresponding
	to $\phi=\boxempty\diamondsuit\mathtt{a}\land\boxempty\diamondsuit\mathtt{b}$
	and an MDP $\mathcal{M}$ as shown in Fig. \ref{fig:LDBA =000026 MDP =000026 product}
	(a) and (b), respectively. For ease of presentation, partial structures
	of the product MDPs $\mathcal{P}=\mathcal{M}\otimes\mathcal{A}_{\phi}$
	and $\mathcal{R}=\mathcal{M}\otimes\mathcal{A}_{\phi}$ are shown
	in Fig. \ref{fig:LDBA =000026 MDP =000026 product} (c) and (d), respectively.
	Since the LTL formula $\phi$ is infeasible over $\mathcal{M}$, there
	is no AMEC in $\mathcal{P}$, whereas there exists one in $\mathcal{R}$.
	Note that there is no $\epsilon$-transitions in this
	case.
\end{example}

Given an accepting path $\boldsymbol{x}_{\infty}^{\boldsymbol{\boldsymbol{\pi}}}=x_{0}\ldots x_{i}x_{i+1}\ldots$, we propose to regulate the multi-objective optimization objective consisting of implementation cost and violation cost for each transition as:
\begin{equation}
c^{\mathcal{R}}\left(x_{i},u_{i}^{\mathcal{R}},x_{i+1}\right)= c_{A}^{\mathcal{R}}\left(x_{i},u_{i}^{\mathcal{R}}\right)\cdot\max\left\{e^{\beta c_{V}^{\mathcal{R}}(x_{i},u_{i}^{\mathcal{R}},x_{i+1})},1\right\} 
\label{eq:combined cost}
\end{equation}
where $\beta\in\mathbb{R}^{+}$ indicates the relative importance.
Based on (\ref{eq:combined cost}), the expected
average regulation cost per stage (ARPS) of $\mathcal{R}$ under a policy $\boldsymbol{\pi}$ is formulated as:
\begin{equation}
   J(\mathcal{R}^{\boldsymbol{\pi}})=\mathbb{E}^{\boldsymbol{\pi}}_{\mathcal{R}}\left[\underset{n\rightarrow\infty}{\lim\inf}\frac{1}{n}\stackrel[i=0]{n}{\sum}c^{\mathcal{R}}\left(x_{i},u_{i}^{\mathcal{R}},x_{i+1}\right)\right]. \label{eq:ARPS}
\end{equation}

In this work, we aim at generating the optimal policy $\boldsymbol{\pi}$ that minimizes the ARPS $J(\mathcal{R}^{\boldsymbol{\pi}})$, while satisfying the acceptance condition of $\mathcal{R}$.

\begin{lem}
\label{lem:priority}
 By selecting a large parameter $\beta>>1$ of (\ref{eq:combined cost})  the first priority of minimizing the AVPS $J_{V}(\mathcal{R}^{\boldsymbol{\pi}})$ in ARPS  $J(\mathcal{R}^{\boldsymbol{\pi}})$ is guaranteed such that minimizing the  AEPS with weighting $\beta$
will never come at the expense of minimizing AVPS i.e., $J_{V}(\mathcal{R}^{\boldsymbol{\pi}})\geq J_{V}(\mathcal{R}^{\boldsymbol{\pi}'})\Longrightarrow J(\mathcal{R}^{\boldsymbol{\pi}})\geq J(\mathcal{R}^{\boldsymbol{\pi}'})$.
\end{lem}
Lemma \ref{lem:priority} can be directly verified based on formulating the exponential function in (\ref{eq:combined cost}).

\begin{problem}
\label{problem2}
Given an $\mathcal{R}$ from $\mathcal{M}$ and $A_{\phi}$, Problem \ref{Prob1} can be formulated
as
\begin{equation}
\begin{aligned}\underset{\boldsymbol{\pi}\in\bar{\boldsymbol{\pi}}}{\min} & J(\mathcal{R}^{\boldsymbol{\pi}})\\
\text{s.t.} & \sideset{}{_{\mathcal{M}}^{\boldsymbol{\pi}}}\Pr\left(\oblong\lozenge\Acc\right)\geq\gamma
\end{aligned}
\label{eq:Prob2}
\end{equation}
where $\beta>>1$, $\bar{\boldsymbol{\pi}}$ represents the set of admissible policies over
$\mathcal{R}$, $\sideset{}{_{\mathcal{M}}^{\boldsymbol{\pi}}}\Pr\left(\oblong\lozenge\Acc\right)$
is the probability of visiting the accepting states of $\mathcal{R}$
infinitely often, and $\gamma$ is the desired threshold for the probability of task satisfaction. 
\end{problem}

\begin{rem}
When the LTL task with respect to the PL-MDP is infeasible, the threshold $\gamma$ represents the probability of entering into any of an AMEC in $\mathcal{R}$. Furthermore, for the cases where there exist no policies satisfying $\sideset{}{_{\mathcal{M}}^{\boldsymbol{\pi}}}\Pr\left(\oblong\lozenge\Acc\right)\geq\gamma$ for a given $\gamma$, the above optimization Problem \ref{problem2} is infeasible, and returns no solutions. However, we can regard $\gamma$ as a slack variable, and technical details are explained in remark \ref{rem:slack}.
\end{rem}

\section{Solution\label{sec:Solution}}

The prefix-suffix structure of LTL satisfaction over an infinite horizon is inspired by the following Lemma

\begin{lem}\label{lem:induced_markov_chain}
Given any Markov
chain $MC_{\mathcal{P}}^{\boldsymbol{\pi}}$ under policy $\boldsymbol{\pi}$, its states can be represented by a disjoint union of a transient
class $\ensuremath{\mathcal{T}_{\boldsymbol{\pi}}}$ and $n_R$ closed
irreducible recurrent classes $\ensuremath{\mathcal{R}_{\boldsymbol{\pi}}^{j}}$,
$j\in\left\{ 1,\ldots,n_{R}\right\} $ \cite{Durrett1999}.
\end{lem}

Given any policy, Lemma \ref{lem:induced_markov_chain} indicates that the behaviors before entering into AMECs involve the transient class, and a recurrent class represents the decision-making within an AMEC. Note that Lemma \ref{lem:induced_markov_chain} provides a general form of state partition that can be applied to any MDP model. This section shows how to integrate the state partition with a relaxed product MDP.
Especially, we analyze states partition to divide Problem \ref{problem2} into two parts and focus on synthesizing the optimal prefix and suffix policies via linear programming (LP), which addresses the trade-off between minimizing the ARPS (Section \ref{subsec:Plan-Suffix}) over a long term and reaching the probability threshold of task satisfaction. 

This solution framework mainly focuses on adopting the ideas of prefix-suffix plans and the method of MDP optimization for relaxed product structures. The details about the intuition, i.e., the analysis of policies over an infinite horizons, computation of AMECs, and linear programming, can be found in
 \cite{Baier2008, Ding2014a, Forejt2011}.

\subsection{State Partition\label{subsec:state}}

According to Property 1 of Theorem \ref{Thm_Properties}, let
$\Xi_{i}^{\mathcal{R}}=\left(X_{i},U_{i}^{\mathcal{R}}\right)$ denote
an AMEC of $\mathcal{R}$ and let $\Xi_{acc}^{\mathcal{R}}=\left\{ \Xi_{1}^{\mathcal{R}},\ldots,\Xi_{N}^{\mathcal{R}}\right\} $
denote the set of AMECs.
To facilitate the analysis, the state $X$ of $\mathcal{R}$ is divided
into a transient class $X_{T}$ and a recurrent class $X_{R}$, where
$X_{R}=\cup_{\left(X_{i},U_{i}^{\mathcal{R}}\right)\subseteq\Xi_{acc}^{\mathcal{R}}}X_{i}$
is the union of the AMEC states of $\mathcal{R}$ and $X_{T}=X\setminus X_{R}$.
Let $X_{r}\subseteq X_{T}$ and $X_{\lnot r}\subseteq X_{T}$ denote
the set of states that can and cannot be reached from the initial
state $x_{0}$, respectively. Since the states in $X_{\lnot r}$ cannot
be reached from $x_{0}$ (i.e., bad states), we will only focus on
$X_{r}$, which can be further divided into $X_{n}$ and $X_{\lnot n}$
based on the violation conditions. Let $X_{\lnot n}$ and $X_{n}$ be
the set of states that can reach $X_{R}$ with and without violation
edges, respectively. Based on $X_{n}$ , $X_{\lnot n}$ and $X_{R}$,
let $X_{tr},X_{tr}'\subseteq X_{R}$ denote the sets of states that
can be reached within one transition from $X_{n}$ and $X_{\lnot n}$,
respectively. An example is provided in Fig. \ref{fig:graph} to illustrate
the partition of states. 

\subsection{Plan Prefix\label{subsec:Plan-Prefix}}

\begin{figure}[t]
	\centering{}\includegraphics[scale=0.6]{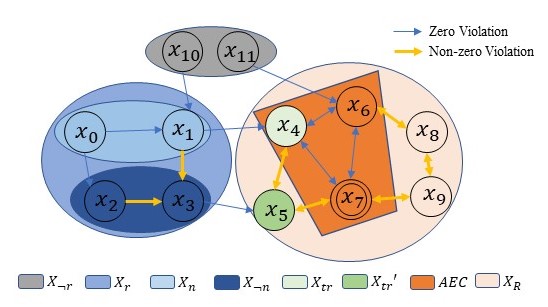}\caption{\label{fig:graph} The illustration of the partition of $X$ in $\ensuremath{\mathcal{R}}$,
		where $x_{7}$ is an accepting state. The edges with and without violation
		cost are marked.}
\end{figure}

The objective of plan prefix is to construct an optimal policy that 
drives the agent from $x_{0}$ to $X_{tr}\cup X_{tr}'$ while minimizing
the combined average cost. To achieve this goal, we construct the prefix MDP model of $\mathcal{R}$ to analyze the prefix behaviors under any policy.

\begin{defn}
\label{def:prefix}
A prefix MDP of $\mathcal{R}$ can be defined as $\ensuremath{\mathcal{R}}_{pre}=\left\{ X_{pre},U_{pre}^{\mathcal{R}},p_{pre}^{\mathcal{R}},x_{0},c_{A_{pre}}^{\mathcal{R}},c_{V_{pre}}^{\mathcal{R}}\right\} $
of $\ensuremath{\mathcal{R}}$, where 
\begin{itemize}
    \item $X_{pre}=X_{r}\cup X_{tr}\cup X_{tr}'\cup v$,
where $v$ is a trap state that models the behaviors within the union of AMECs.

	$\text{ }$
	
    \item The set of actions is $U_{pre}^{\mathcal{R}}=U^{\mathcal{R}}\cup\tau$, 
where $\tau$ represents a self-loop action only enabled at state $v$ s.t. $\tau=U_{pre}(v)$, and $U_{pre}^{\mathcal{R}}(x)$ is the actions enabled at $x\in X_{pre}$.

	$\text{ }$
	
    \item The transition probability $p_{pre}^{\mathcal{R}}$,
is defined as: (i) $p_{pre}^{\mathcal{R}}(x,u^{\mathcal{R}},\bar{x})=p^{\mathcal{R}}(x,u^{\mathcal{R}},\bar{x})$,
$\forall x\in X_{r}$, $\bar{x}\in X_{pre}\setminus  v$, and $\forall u^{\mathcal{R}}\in U^{\mathcal{R}}\left(x\right)$;
(ii) $p_{pre}^{\mathcal{R}}(x,u^{\mathcal{R}},v)=1$, $\forall x\in X_{tr}\cup X_{tr}'$, $u^{\mathcal{R}}\in U^{\mathcal{R}}\left(x\right)$ and $v\in\mathcal{V}$; (iii)
$P_{pre}^{\mathcal{R}}(v,\tau,v)=1$,

	$\text{ }$
	
    \item The implementation cost is defined
as: (i) $c_{A_{pre}}^{\mathcal{R}}(x,u^{\mathcal{R}})=c_{A}^{\mathcal{R}}(x,u^{\mathcal{R}})$,
$\forall x\in X_{r}$ and $u^{\mathcal{R}}\in U^{\mathcal{R}}\left(x\right)$;
and (ii) $c_{A_{pre}}^{\mathcal{R}}(x,u^{\mathcal{R}})=c_{A_{pre}}^{\mathcal{R}}(v,\tau)=0$,
$\forall x\in X_{tr}\cup X_{tr}'$,

	$\text{ }$
	
    \item The violation cost is defined
as: $c_{V_{pre}}^{\mathcal{R}}(x,u^{\mathcal{R}},\bar{x})=c_{V}^{\mathcal{R}}(x,u^{\mathcal{R}},\bar{x})$,
$\forall x\in X_{r}$, $\bar{x}\in X_{r}\cup X_{tr}$, $u^{\mathcal{R}}\in U^{\mathcal{R}}\left(x\right)$;
$c_{V_{pre}}^{\mathcal{R}}(x,u^{\mathcal{R}},\bar{x})=0$ otherwise.
\end{itemize}
\end{defn}

In Def. \ref{def:prefix}, $v$ is the trap state s.t. there's only self-loop action enabled at the state. The agent's state remains the same once it enters the trap state. Therefore, the optimization process of desired policy over prefix product MDP $\mathcal{R}_{pre}$ can be formulated as

\begin{equation}
\begin{array}{cc}\underset{\boldsymbol{\pi}\in\boldsymbol{\bar{\pi}}_{pre}}{\min} & \mathbb{E}_{\ensuremath{\mathcal{R}}_{pre}}^{\boldsymbol{\pi}}\left[\underset{n\rightarrow\infty}{\lim\sup}\frac{1}{n}\stackrel[i=0]{n}{\sum}c_{pre}^{\mathcal{R}}\left(x_{i},u_{i}^{\mathcal{R}},x_{i+1}\right)\right]\\
\text{s.t.} & \sideset{}{_{x_{0},\ensuremath{\mathcal{R}}_{pre}}^{\boldsymbol{\pi}}}\Pr(\lozenge v)\geq\gamma,
\end{array}
\label{eq:Prob_pre}
\end{equation}

where $\boldsymbol{\bar{\pi}}_{pre}$
represents a set of all admissible policies over $\mathcal{R}_{pre}$, $\sideset{}{_{x_{0},\ensuremath{\mathcal{R}}_{pre}}^{\boldsymbol{\pi}}}\Pr\left(\lozenge v\right)$
denotes the probability of $\boldsymbol{x}_{pre}$ starting from
$x_{0}$ and eventually reaching the trap state $v$, and $c_{pre}^{\mathcal{R}}(x_{i},u_{i}^{\mathcal{R}},x_{i+1})$ is the regulation cost for each transition such that

\begin{equation}
\begin{array}{cc}
c_{pre}^{\mathcal{R}}(x_{i},u_{i}^{\mathcal{R}},x_{i+1})=c_{A_{pre}}^{\mathcal{R}}(x_{i},u_{i}^{\mathcal{R}})\\
\cdot\max\left\{ e^{\beta\cdot c_{V_{pre}}^{\mathcal{R}}(x_{i},u_{i}^{\mathcal{R}},x_{i+1})},1\right\}.\\
\end{array}
\end{equation}

In the prefix plan, the policies of staying within AMECs can be modeled by adding the trap state $v$. Based on the station partition in Section \ref{subsec:state}, reaching an AMEC of $\ensuremath{\mathcal{R}}$ is equivalent to reaching the set $X_{tr}\cup X_{tr}'$. Furthermore, since there exist policies under which paths starting from $X_{r}$ to $X_{tr}\cup X_{tr}'$ only traverse the transitions with zero violation cost and the cost of staying at $v$ 
is zero, a large $\beta$ in $c^{\mathcal{R}}_{pre}$
is employed to search policies minimizing the AVPS over $\ensuremath{\mathcal{R}}_{pre}$ as the first priority. 
It should
be noted that there always exists at least one solution $\boldsymbol{\pi}$ in (\ref{eq:Prob_pre}). This is because AMECs in $\ensuremath{\mathcal{R}}$
always exist by Theorem \ref{Thm_Properties}, and we can always obtain a valid prefix MDP $\ensuremath{\mathcal{R}}_{pre}$ of $\mathcal{R}$.

Inspired by the network flow approaches
\cite{Forejt2011,Forejt2012}, (\ref{eq:Prob_pre}) can be reformulated as a graph-constrained optimization problem and solved through LP. Especially, let $y_{x,u}$ denote the
expected number of times over the infinite horizons such that $x$ is visited with $u\in U_{pre}^{\mathcal{R}}$. It measures the state occupancy
among all paths starting from the initial state $x_{0}$ under policy
$\boldsymbol{\pi}$ in $\ensuremath{\mathcal{R}}_{pre},$ i.e., $y_{x,u}=\stackrel[i=0]{\infty}{\sum}\sideset{}{_{x_{0},\ensuremath{\mathcal{R}}_{pre}}^{\boldsymbol{\pi}}}\Pr\left(x_{i}=x,u_{i}^{\ensuremath{\mathcal{R}}}=u\right)$
. Then, we can solve (\ref{eq:Prob_pre}) as the following LP: 

\begin{equation}
    \begin{array}{cc}
    \underset{\left\{ y_{x,u}\right\} }{\min}\left[J_{pre}\overset{\vartriangle}{=}\underset{\left(x,u\right)}{\sum}\underset{\bar{x}\in X_{pre}}{\sum} y_{x,u}\cdot p_{pre}^{\mathcal{R}}\left(x,u,\bar{x}\right)\cdot c_{pre}^{\mathcal{R}}\left(x,u,\bar{x}\right)\right]\\
    \text{  }\\
    \text{s.t. } \underset{\left(x,u\right)}{\sum}\underset{\bar{x}\in (X_{tr}\cup X_{tr}')}{\sum} y_{x,u}\cdot p_{pre}^{\mathcal{R}}\left(x,u,\bar{x}\right)\geq\gamma\\
    \text{  }\\
    \underset{u'\in U_{pre}^{\mathcal{R}}\left(x'\right)}{\sum}y_{x',u'}= \underset{\left(x,u\right)}{\sum} y_{x,u}\cdot p_{pre}^{\mathcal{R}}(x,u,x') +\chi_{0}\left(x'\right)\\
    \text{  }\\
    y_{x,u}\geq0, \forall x'\in X_{r}\\
    \end{array}
\label{eq: Pre_LP}
\end{equation}

where $\chi_{0}$ is the distribution of initial state, and $\underset{\left(x,u\right)}{\sum}\coloneqq\underset{x\in (X_{r}\cup X_{tr}\cup X_{tr}')}{\sum}\underset{u\in U_{pre}^{\mathcal{R}}\left(x\right)}{\sum}$.

Once the solution $y_{x,u}^{*}$ to (\ref{eq: Pre_LP}) is obtained,
the optimal stochastic policy $\boldsymbol{\pi}_{pre}^{*}$ can be generated as

\begin{equation}
\boldsymbol{\pi}_{pre}^{*}(x,u)=\left\{ \begin{array}{cc}
\frac{y_{x,u}^{*}}{\underset{\overline{u}\in U_{pre}^{\mathcal{R}}\left(x\right)}{\sum}y_{x,\overline{u}}^{*}} & \text{ if \ensuremath{x\in X_{r}^{*},} }\\
\frac{1}{\bigl|U_{pre}^{\mathcal{R}}\left(x\right)\bigr|} & \text{ if \ensuremath{x\in X_{pre}\setminus X_{r}^{*}},}
\end{array}\right.\label{eq:prefix_policy}
\end{equation}

where $X_{r}^{*}=\left\{ x\in X_{r}\left|\underset{u\in U_{pre}^{\mathcal{R}}\left(x\right)}{\sum}y_{x,u}^{*}>0\right.\right\} $.

\begin{lem}
	\label{lemma:prefix}The optimal policy $\boldsymbol{\pi}_{pre}^{*}$ in (\ref{eq:prefix_policy})
	ensures that $\sideset{}{_{x_{0},\ensuremath{\mathcal{R}}_{pre}}^{\boldsymbol{\pi}}}\Pr\left(\lozenge v\right)\geq\gamma$.
\end{lem}
\begin{proof}
	The proof is similar to Lemma 3.3 in \cite{Etessami2007}. Due to
	the transient class of $X_{r}$, $y_{x,u}$ is finite. In first constraint of (\ref{eq: Pre_LP}), the sum $\underset{\left(x,u\right)}{\sum}\underset{\bar{x}\in (X_{tr}\cup X_{tr}')}{\sum} y_{x,u}\cdot p_{pre}^{\mathcal{R}}\left(x,u,\bar{x}\right)$
	is the expected number of times that $X_{tr}\cup X_{tr}'$
	can be reached for the first time from a given initial state under the policy $\boldsymbol{\pi}_{pre}^{*}$.
	Since the agent remains in $v$ once it enters $X_{tr}\cup X_{tr}'$,
	the sum is the probability of reaching $X_{R}$, which is lower bounded
	by $\gamma$. The second constraint of (\ref{eq: Pre_LP})
guarantees the balances of network flow for the distribution of initial states.
\end{proof}

\begin{example}
	\begin{figure}
	\centering{}\includegraphics[scale=0.28]{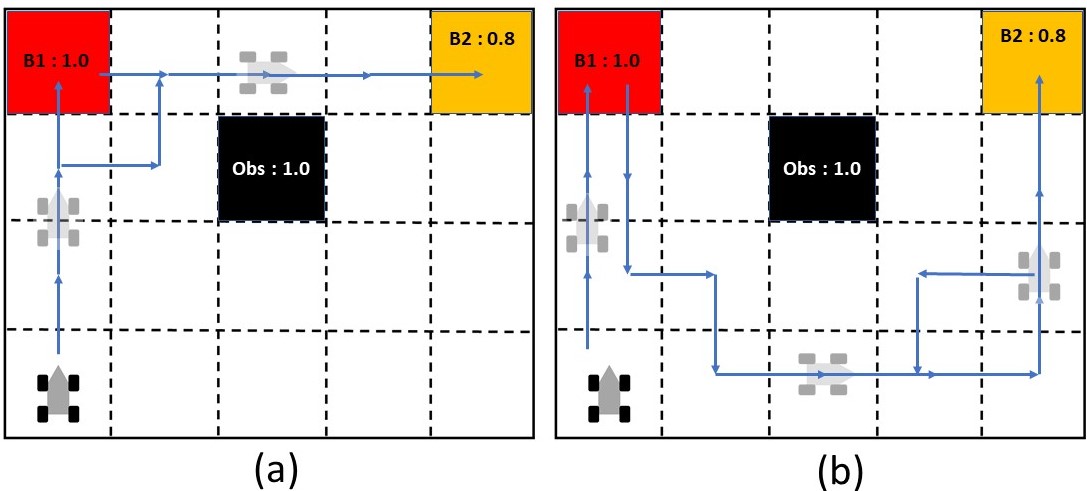}\caption{\label{example_prefix} Simulated trajectories with different $\gamma$ under optimal prefix policies. (a) $\gamma=0.3$. (b) $\gamma=1.0$.}
\end{figure}
    As a running example in Fig. \ref{example_prefix}, we illustrate the importance of the threshold $\gamma$ that balances the trade-off between optimizing the ARPS and reaching the probability of satisfaction. The motion uncertainties and the action cost are set the same as in Example \ref{example1}. An LTL task is considered as $\phi_{pre}=\Eventually (\mathtt{B1}\land\Eventually\mathtt{B2})$ that requires visiting region $\mathtt{B1}$ first and then $\mathtt{B2}$ sequentially. $\phi_{pre}$ is feasible with respect to the corresponding PL-MDP ($\mathtt{B2}$ is surrounded by probabilistic obstacles). Fig. \ref{example_prefix} shows the results with two different $\gamma$ under generated prefix optimal policies. It can be observed how such a parameter impacts the optimization bias since $\gamma$ represents the quantitative probabilistic satisfaction \cite{Baier2008}. 
\end{example}

\begin{rem}
\label{rem:slack}
The pre-defined threshold may influence the feasibility of the optimization (\ref{eq: Pre_LP}) when there exist no policies s.t. $\sideset{}{_{x_{0},\ensuremath{\mathcal{R}}_{pre}}^{\boldsymbol{\pi}}}\Pr(\lozenge v)\geq\gamma$. Since LP is a linear convex optimization \cite{Forejt2011,Forejt2012}, to alleviate the issue, we can treat $\gamma$ as a slack variable such that (\ref{eq: Pre_LP}) can be reformulated as: 
\begin{equation}
\begin{array}{cc}\underset{\boldsymbol{\pi}\in\boldsymbol{\bar{\pi}}_{pre}}{\min} & \mathbb{E}_{\ensuremath{\mathcal{R}}_{pre}}^{\boldsymbol{\pi}}\left[\underset{n\rightarrow\infty}{\lim\sup}\frac{1}{n}\stackrel[i=0]{n}{\sum}c_{pre}^{\mathcal{R}}\left(x_{i},u_{i}^{\mathcal{R}},x_{i+1}\right)\right] + w_{\gamma}\cdot\gamma\\
\text{s.t.} & \sideset{}{_{x_{0},\ensuremath{\mathcal{R}}_{pre}}^{\boldsymbol{\pi}}}\Pr(\lozenge v)\geq\gamma,
\end{array}
\label{eq:Prob_pre_slack}
\end{equation}
where $w_{\gamma}$ is a regulation parameter that can be designed based on the users' preference for the multi-objective problems. Then, directly adopting the optimization process as the same as (\ref{eq: Pre_LP}) and (\ref{eq:prefix_policy}) can find feasible solutions. 

Because the optimization of (\ref{eq:Prob_pre_slack}) increases the computational complexity, it is only applied when the formulation (\ref{eq: Pre_LP}) returns no solution.
\end{rem}

\subsection{Plan Suffix\label{subsec:Plan-Suffix}}

Suppose the prefix optimal policy $\boldsymbol{\pi}_{pre}^{*}$ drives the RL-agent into one AMEC $\Xi_{j}^{\mathcal{R}}$ of $\Xi_{acc}^{\mathcal{R}}$.
This section considers the long-term
behavior of the agent inside AMEC $\Xi_{j}^{\mathcal{R}}$. Since the agent can enter any AMEC, let $\Xi_{j}^{\mathcal{R}}=\left(X_{j},U_{j}^{\mathcal{R}}\right)\subseteq\Xi_{acc}^{\mathcal{R}}$
denote such an AMEC and $X_{j}^{tr}= X_{j}\cap(X_{tr}\cup X_{tr}')$ denote the
set of states that can be reached from plan prefix $\boldsymbol{x}_{pre}$.
As a result, $X_{j}^{tr}$ can be treated as an initial state for
plan suffix after entering the AMEC. The objective of suffix policies is to enforce the accepting conditions and consider the optimization of long-term behavior. Therefore, after the agent entering into AMEC $\Xi_{j}^{\mathcal{R}}$, the optimization process of the desired policy over $\Xi_{j}^{\mathcal{R}}$ can be formulated: 

\begin{equation}
\begin{array}{cc}
     \underset{\boldsymbol{\pi}\in\boldsymbol{\bar{\pi}}_{\Xi_{j}^{\mathcal{R}}}}{\min} \mathbb{E}_{\Xi_{j}^{\mathcal{R}}}^{\boldsymbol{\pi}}\left[\underset{n\rightarrow\infty}{\lim\sup}\frac{1}{n}\stackrel[i=0]{n}{\sum}c^{\mathcal{R}}\left(x_{i},u_{i}^{\mathcal{R}},x_{i+1}\right)\right]  \\
     \text{s.t.   } \inf(\boldsymbol{x}_{\Xi_{j}^{\mathcal{R}}}^{\boldsymbol{\pi}})\cap\Acc\neq\emptyset, \forall \boldsymbol{x}_{\Xi_{j}^{\mathcal{R}}}^{\boldsymbol{\pi}}\in \boldsymbol{X}_{\Xi_{j}^{\mathcal{R}}}^{\boldsymbol{\pi}},
\end{array}
\label{eq:formulate_suffix}
\end{equation}
where $\boldsymbol{\bar{\pi}}_{\Xi_{j}^{\mathcal{R}}}$
represents a set of all admissible policies over $\Xi_{j}^{\mathcal{R}}$, $\boldsymbol{X}_{\Xi_{j}^{\mathcal{R}}}^{\boldsymbol{\pi}}$ is the set of all paths over the finite horizons under the policy $\boldsymbol{\pi}$, and $c^{\mathcal{R}}$ is the regulation transition cost in (\ref{eq:combined cost}).

Let $A_{j}$ denote a set of accepting states in $X_{j}$ of $\Xi_{j}^{\mathcal{R}}$, i.e., $A_{j}=X_{j}\cap\Acc$. Consequently, the acceptance condition of $\mathcal{R}$ can be satisfied s.t. $\inf(\boldsymbol{x}_{\Xi_{j}^{\mathcal{R}}}^{\boldsymbol{\pi}})\cap\Acc\neq\emptyset$. One infinite accepting path $\boldsymbol{x}_{\Xi_{j}^{\mathcal{R}}}^{\boldsymbol{\pi}}$ can be regarded as a concatenation of an infinite number of cyclic paths starting and ending in the set $A_{j}$.

\begin{defn}
	\label{def:average cycle cost} A cyclic path $\boldsymbol{x}_{c}=x_{1}\ldots x_{N_{x_{c}}}$
	associated with $\Xi_{j}^{\mathcal{R}}$ is a finite path with horizons
	$N_{x_{c}}$ starting and ending at any subset $X_{j}'\subseteq X_{j}$, i.e., $x_{1},x_{N_{x_{c}}}\in X_{j}'$,
	while actions are restricted to $U_{j}^{\mathcal{R}}$ to remain within $X_{j}$. A cyclic path $\boldsymbol{x}_{c}$
	is called an accepting cyclic path if it starts and ends at $A_{j}$, i.e., $x_{1},x_{N_{x_{c}}}\in A_{j}$.
\end{defn}

By definition \ref{def:average cycle cost}, the optimization problem (\ref{eq:formulate_suffix}) over the infinite horizons can be reformulated as

\begin{equation}
\begin{array}{cc}
     \underset{\boldsymbol{\pi}\in\boldsymbol{\bar{\pi}}_{\Xi_{j}^{\mathcal{R}}}}{\min}\mathbb{E}_{\boldsymbol{x}_{c}\in \boldsymbol{X}_{\Xi_{j}^{\mathcal{R}},cycle}^{\boldsymbol{\pi}}}\left[\frac{1}{N_{x_{c}}}\stackrel[i=1]{N_{x_{c}}}{\sum}\left(c^{\mathcal{R}}\left(x_{i},u_{i}^{\mathcal{R}},x_{i+1}\right)\right)\right]  \\
     \text{s.t.   } x_{1}, x_{N_{x_{c}}}\in A_{j}, \forall \boldsymbol{x}_{c}\in \boldsymbol{X}_{\Xi_{j}^{\mathcal{R}},cycle}^{\boldsymbol{\pi}}
\end{array}
\label{eq:suffix_prob}
\end{equation}
where $\boldsymbol{X}_{\Xi_{j}^{\mathcal{R}},cycle}^{\boldsymbol{\pi}}$ is a set of all cyclic paths under policy $\boldsymbol{\pi}$ over the AEMC $\Xi_{j}^{\mathcal{R}}$, the mean cyclic cost $\frac{1}{\bar{N}}\stackrel[i=t]{t+\bar{N}}{\sum}\left(c^{\mathcal{R}}\left(x_{i},u_{i}^{\mathcal{R}},x_{i+1}\right)\right)$ corresponds to the average cost per stage, and the constraint requires all induced cyclic paths under policy $\pi$ are accepting cyclic paths. 

Similarly, inspired from
\cite{Randour(1)2015,chatterjee2011}, we can also construct the suffix MDP model of $\mathcal{R}$ based on the state partition. Then
(\ref{eq:suffix_prob})
can be solved through LP. In order to apply the network flow algorithm to constrain paths starting
from and ending at $A_{j}$, we need to transform accepting cyclic paths into the form of acyclic paths. 
To do so,  
we split
$A_{j}$ to create a virtual copy $A_{j}^{out}$ that has no incoming
transitions from $A_{j}$ and a virtual copy $A_{j}^{in}$ that only
has incoming transitions from $A_{j}$, which allows representing
a cyclic path as an acyclic path starting from $A_{j}^{out}$ and ending in
$A_{j}^{in}$. To convert the analysis of cyclic paths into acyclic paths, we construct the following suffix MDP of $\mathcal{R}$ for $\Xi_{j}^{\mathcal{R}}$.

\begin{defn}
\label{def:suffix}
 A suffix MDP of $\mathcal{R}$ can be defined as $\ensuremath{\mathcal{R}}_{suf}=\left\{ X_{suf},U_{suf}^{\mathcal{R}},p_{suf}^{\mathcal{R}},D_{tr}^{\mathcal{R}},c_{A_{suf}}^{\mathcal{R}},c_{V_{suf}}^{\mathcal{R}}\right\} $
where
\begin{itemize}
    \item $X_{suf}=(X_{j}\backslash A_{j})\cup A_{j}^{out}\cup A_{j}^{in}$.
    \item  $U_{suf}^{\mathcal{R}}=U_{j}^{\mathcal{R}}\cup\tau$
with $U_{suf}^{\mathcal{R}}\left(x\right)=\tau,\forall x\in A_{j}^{in}$, where $U_{suf}^{\mathcal{R}}(x)$ is the actions enabled at $x\in X_{suf}$.

	$\text{ }$
	
    \item The transition probability $p_{suf}^{\mathcal{R}}$ can be defined
as follows: (i) $p_{suf}^{\mathcal{R}}(x,u^{\mathcal{R}},\overline{x})=p^{\mathcal{R}}(x,u^{\mathcal{R}},\overline{x})$,
$\forall x,\overline{x}\in (X_{j}\backslash A_{j})\cup A_{j}^{out}$
and $u^{\mathcal{R}}\in U_{j}^{\mathcal{R}}$; (ii) $p_{suf}^{\mathcal{R}}(x,u^{\mathcal{R}},\overline{x})=p^{\mathcal{R}}(x,u^{\mathcal{R}},\overline{x})$,$\forall x\in (X_{j}\backslash A_{j})\cup A_{j}^{out}$,
$\overline{x}\in A_{j}^{in}$ and $u^{\mathcal{R}}\in U_{j}^{\mathcal{R}}$;
(iii) $p_{suf}^{\mathcal{R}}(x,\tau,\overline{x})=1$, $\forall x,\overline{x}\in A_{j}^{in}$.

	$\text{ }$
	
    \item The implementation cost is defined as: (i) $c_{A_{suf}}^{\mathcal{R}}(x,u^{\mathcal{R}})=c_{A}^{\mathcal{R}}(x,u^{\mathcal{R}})$,
$\forall x\in X_{j}\backslash A_{j}\cup A_{j}^{out}$ and $u^{\mathcal{R}}\in U_{j}^{\mathcal{R}}$;
and (ii) $c_{A_{suf}}^{\mathcal{R}}(x,\tau)=0$, $\forall x\in A_{j}^{in}$.

	$\text{ }$
	
    \item The violation cost is defined as: (i) $c_{V_{suf}}^{\mathcal{R}}(x,u^{\mathcal{R}},\overline{x})=c_{V}^{\mathcal{R}}(x,u^{\mathcal{R}},\overline{x})$,
$\forall x,\overline{x}\in (X_{j}\backslash A_{j})\cup A_{j}^{out}$
and $u^{\mathcal{R}}\in U_{j}^{\mathcal{R}}$; (ii) $c_{V_{suf}}^{\mathcal{R}}(x,u^{\mathcal{R}},\overline{x})=c_{V}^{\mathcal{R}}(x,u^{\mathcal{R}},\overline{x})$,$\forall x\in (X_{j}\backslash A_{j})\cup A_{j}^{out}$,
$\overline{x}\in A_{j}^{in}$ and $u^{\mathcal{R}}\in U_{j}^{\mathcal{R}}$;
and (iii) $c_{V_{suf}}^{\mathcal{R}}(x,\tau,\overline{x})=0$,$\forall x,\overline{x}\in A_{j}^{in}$.

	$\text{ }$
	
    \item The distribution of the initial state $D_{tr}^{\mathcal{R}}:X_{suf}\shortrightarrow\mathbb{R}$
is defined as (i) $D_{tr}^{\mathcal{R}}\left(x\right)=\underset{\hat{x}\in X_{n}}{\sum}\underset{u^{\mathcal{R}}\in U^{\mathcal{R}}}{\sum}y_{pre}^{*}(\hat{x},u^{\mathcal{R}})\cdot P^{\mathcal{R}}(\hat{x},u^{\mathcal{R}},x)$, if $x\in X_{j}^{tr}$; (ii) $D_{tr}^{\mathcal{R}}\left(x\right)=0$
if $x\in X_{suf}\backslash X_{j}^{tr}$, where $X_{n}$ is a set of
states that can reach $X_{R}$ in transient class.  
\end{itemize}
\end{defn}
Let $z_{x,u}=z\left(x,u\right)$ denote the long-term frequency that
the state is at $x\in X_{suf}\setminus A_{j}^{in}$ and the action
$u$ is taken. Then, to solve (\ref{eq:suffix_prob}), the following
LP is formulated as

\begin{equation}
    \begin{array}{cc}
    \underset{\left\{ z_{x,u}\right\} }{\min} \left[J_{\Xi_{j}^{\mathcal{R}}}\overset{\vartriangle}{=}\underset{\left(x,u\right)}{\sum}\underset{\overline{x}\in X_{suf}}{\sum}z_{x,u}p_{suf}^{\mathcal{R}}\left(x,u,\bar{x}\right)c_{suf}^{\mathcal{R}}\left(x,u,\bar{x}\right)\right]\\
    \text{  }\\
    \text{s.t. }\underset{u'\in U_{suf}^{\mathcal{R}}\left(x'\right)}{\sum}z_{x',u'}= {\underset{\left(x,u\right)}{\sum}}z_{x,u}\cdot p_{suf}^{\mathcal{R}}(x,u,x')+D_{tr}^{\mathcal{R}}\left(x'\right)\\
    \text{  }\\
    \underset{\left(x,u\right)}{\sum}\underset{\overline{x}\in A_{j}^{in}}{\sum}z_{x,u}\cdot p_{suf}^{\mathcal{R}}(x,u,\bar{x})=\underset{x\in X_{suf}'}{\sum}D_{tr}^{\mathcal{R}}\left(x\right),\\
    \text{  }\\
    z_{x,u}\geq0, \forall x'\in X_{suf}'\\
    \end{array}
\label{eq:suffix_LP}
\end{equation}

where $X_{suf}'=X_{suf}\setminus A_{j}^{in}$,
$\underset{\left(x,u\right)}{\sum}\coloneqq\underset{x\in X_{suf}'}{\sum}\underset{u\in U_{suf}^{\mathcal{R}}\left(x\right)}{\sum}$,
and $c_{suf}^{\mathcal{R}}\left(x,u,\bar{x}\right)=c_{A_{suf}}^{\mathcal{R}}\left(x,u\right)\cdot\max\left\{ e^{\beta\cdot c_{V_{suf}}^{\mathcal{R}}\left(x,u,\bar{x}\right)},1\right\} $. The first constraint represents the in-out flow balance, and the second constraint ensures that $A_{j}^{in}$ is eventually reached.
Note that (\ref{eq:formulate_suffix}) and (\ref{eq:suffix_prob}) are defined for the suffix MDP $\Xi_{j}^{\mathcal{R}}$, whereas (\ref{eq:suffix_LP}) is formulated over the suffix MDP $\ensuremath{\mathcal{R}}_{suf}$.

Once the solution $z_{x,u}^{*}$ to (\ref{eq:suffix_LP}) is obtained,
the optimal policy can be generated by 
\begin{equation}
\boldsymbol{\pi}_{suf}^{*}(x,u)=\begin{cases}
\frac{z_{x,u}^{*}}{\underset{\overline{u}\in U_{suf}^{\mathcal{R}}\left(x\right)}{\sum}z_{x,\overline{u}}^{*}}, & \text{if } x\in X_{j}^{*},\\
\frac{1}{\bigl|U_{suf}^{\mathcal{R}}\left(x\right)\bigr|}, & \text{if } x\in X_{suf}\setminus X_{j}^{*}
\end{cases}\label{eq:suffix_policy}
\end{equation}
where $X_{j}^{*}=\left\{ x\in X_{j}\left|\underset{u\in U_{suf}^{\mathcal{R}}\left(x\right)}{\sum}z_{x,u}^{*}>0\right.\right\} $.

\begin{algorithm}
	\caption{\label{Alg1} Synthesis and execution of complete policy }
	
	\scriptsize
	
	\singlespacing
	
	\begin{algorithmic}[1]
		
		\Procedure {Input: } {$\mathcal{M}$ , $\phi$, and $\beta$}
		
		{Output: } { the optimal policy $\boldsymbol{\pi}^{*}$ and $\boldsymbol{\boldsymbol{\boldsymbol{\boldsymbol{\mu}}}}^{*}$ }
		
		{Initialization: } { Construct $\ensuremath{\mathcal{\mathcal{A}_{\phi}}}$
			and $\mathcal{R}=\mathcal{M}\times\mathcal{A}_{\phi}$}
		
		\State Set $t=0$ and the execution horizons $T$. 
		
		\State Construct AMECs $\Xi_{acc}^{\mathcal{R}}=\left\{ \Xi_{1}^{\mathcal{R}},\ldots,\Xi_{N}^{\mathcal{R}}\right\} $. 
		
		\State Construct $X_{r}$,$X_{n}$, $X_{\lnot n}$, $X_{tr},X_{tr}'$.

		\If { $X_{r}=\emptyset$ }
		
		\State $\Xi_{acc}^{\mathcal{\mathcal{R}}}$ can not be reached from
		$x_{0}$ and no $\boldsymbol{\pi}^{*}$ exists;
		
		\Else
		
		\State Construct $\ensuremath{\mathcal{R}}_{pre}$.
		
		\For { each $\Xi_{j}^{\mathcal{R}}\subseteq\Xi_{acc}^{\mathcal{\mathcal{R}}}$
		}
		
		\State Construct $\ensuremath{\mathcal{R}}_{suf}$.
		
		\EndFor
		
		\State Obtain $\boldsymbol{\pi}^{*}$ by solving the coupled LP in \ref{eq: Complete_policy}.
    
		\State Set $x_{t}=x_{0}=\left(s_{0},l_{0},q_{0}\right)$ and $s_{t}=s_{0}$.
		
		\State Set $\boldsymbol{s}_{\mathcal{M}}=x_{t}$
		
		\While { $t\leq T$ }
		
		\State Select an action $u_{t}^{\mathcal{R}}$ according to $\boldsymbol{\pi}^{*}\left(x_{t}\right)$.
		
		\State Obtain $s_{t+1}$ in $\mathcal{M}$ by applying action $a_{t}=u_{t}\bigr|_{\mathcal{\mathcal{M}}}^{\mathcal{R}}$.
		
		\State Observe $l_{t+1}$ .
		
		\State Set $x_{t+1}=\left(s_{t+1},l_{t+1},q_{t+1}\right)$.
		
		\State Update $\boldsymbol{s}$ by concatenating $x_{t+1}$.
		
		\State $t++$.
		
		\EndWhile
		
		\State Return $\boldsymbol{\boldsymbol{\boldsymbol{\boldsymbol{\mu}}}}^{*}\left(\boldsymbol{s}\left[:t\right],L(\boldsymbol{s}\left[:t\right])\right)$
		$\forall t=0,1\ldots T$.
		
		\EndIf
		
		\EndProcedure
		
	\end{algorithmic}
\end{algorithm}

\begin{lem}
	\label{lemma:suffix} The plan suffix $\boldsymbol{\pi}_{suf}^{*}$ in (\ref{eq:suffix_policy})
	solves (\ref{eq:formulate_suffix}) for the suffix MDP of AMEC $\Xi_{j}^{\mathcal{R}}$.
\end{lem}
\begin{proof}
 Due to the fact that all input
	flow from the transient class will eventually end up in $A_{j}^{in}$, the second constraint in (\ref{eq:suffix_LP})
	guarantees the states in $A_{j}^{in}$ can be eventually reached from
	$x\in X_{suf}\backslash A_{j}^{in}$.
	Thus, the
	solution of (\ref{eq:suffix_LP}) indicates the accepting states $A_{j}$
	can be visited infinitely often within the AMEC $\Xi_{j}^{\mathcal{R}}$. Based on the construction of $\ensuremath{\mathcal{R}}_{suf}$, the objective function in (\ref{eq:suffix_LP}) represents the mean
	cost of cyclic paths analyzed in (\ref{eq:suffix_prob}), which is exactly the ARPS of suffix MDP $\Xi_{j}^{\mathcal{R}}$ in (\ref{eq:formulate_suffix}).
\end{proof}

\begin{rem}
The above process that constructing suffix MDP in definition \ref{def:suffix}, solving optimization problem (\ref{eq: Pre_LP}), and synthesizing suffix optimal policies (\ref{eq:suffix_policy}) is repetitively applied to every AMEC $\Xi_{j}^{\mathcal{R}}$ of $\Xi_{acc}^{\mathcal{R}}$. 
\end{rem}

To demonstrate the efficiency of our approach, we apply the widely used Round-Robin policy \cite{Baier2008} for comparison in the following example and Section \ref{sec:Case}.
After the agent enters into one AMEC $\Xi_{i}^{\mathcal{R}}=\left(X_{i},U_{i}^{\mathcal{R}}\right)$, an ordered sequence of actions from $U_{i}^{\mathcal{R}}(x), \forall x\in X_{i}$ is created. The Round-Robin policy guides the agent to visit each state by iterating over the ordered actions, and this ensures all states of the AMEC are visited infinitely often (i.e., satisfying the acceptance condition). For decision-making within an AMEC,  the Round-Robin policy does not consider optimality.

\begin{example}
	\begin{figure}
	\centering{}\includegraphics[scale=0.28]{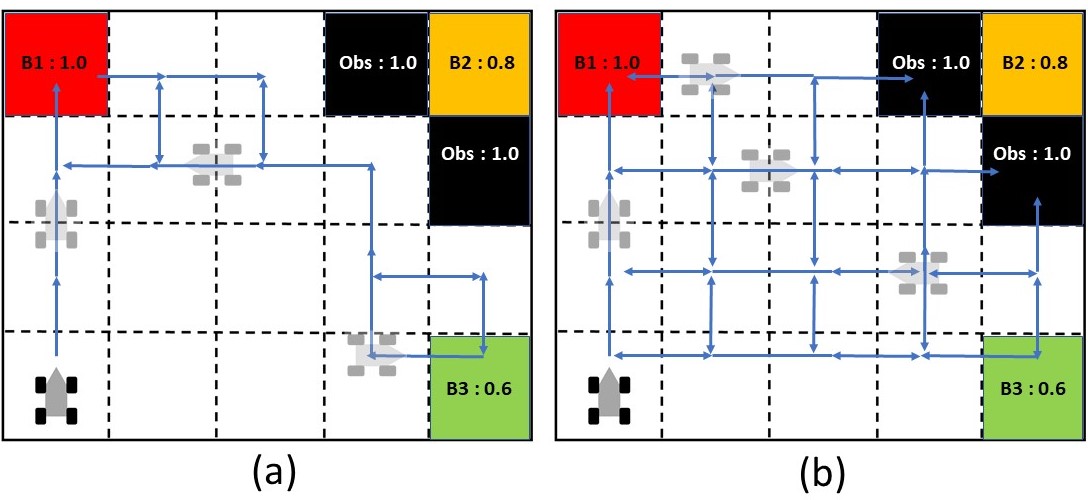}\caption{\label{example_suffix} Simulated trajectories with different suffix policies. (a) Optimal policy generated from our framework. (2) The Round-Robin Policy.}
\end{figure}
    As another running example in Fig. \ref{example_suffix}, we demonstrate the importance of optimizing ARPS to the motion planning after entering into AMECs compared with the Round-Rabin policy.
    The motion uncertainties and the action cost are set as the same as in Example \ref{example1}. An LTL task is considered as $\phi_{suf}=\Always\Eventually\mathtt{B1}\land\Always\Eventually\mathtt{B2}\land\Always\Eventually\mathtt{B3}$ that requires to infinitely often visit regions $\mathtt{B1}$, $\mathtt{B2}$ and $\mathtt{B3}$. $\phi_{suf}$ is infeasible with respect to the corresponding PL-MDP ($\mathtt{B3}$ is surrounded by obstacles). Fig. \ref{example_suffix} shows the results with two different policies. Without minimizing the AVPS, the Round-Rabin policy can not be applied to the infeasible cases using the relaxed product MDP $\mathcal{R}$. In addition, the framework in \cite{Guo2018} returns no solution for this example.
\end{example}

\subsection{Complete Policy and Complexity}

A complete optimal stationary policy $\boldsymbol{\pi}^{*}$ can be obtained by concatenating the procedure of
solving the linear programs in (\ref{eq: Pre_LP}) and (\ref{eq:suffix_LP}) as

\begin{equation}
    \begin{array}{cc}
    \underset{\left\{ y_{x,u}, z_{x,u}\right\} }{\min}\left[(1-\eta)\cdot J_{pre}+\eta\cdot\underset{\Xi_{j}^{\mathcal{R}}\in\Xi_{Acc}^{\mathcal{R}}}{\sum} J_{\Xi_{j}^{\mathcal{R}}}\right]\\
    \text{  }\\
    \text{s.t. } \text{Constraints in } (\ref{eq: Pre_LP}) \text{ and } (\ref{eq:suffix_LP}),\\
    \text{  }\\
    \end{array}
\label{eq: Complete_policy}
\end{equation}
where $J_{pre}, y_{x,u}$  and $J_{\Xi_{j}^{\mathcal{R}}}, z_{x,u}$ are defined in (\ref{eq: Pre_LP}) and (\ref{eq:suffix_LP}) respectively, and $\eta$ is a trade-off parameter to balance the importance of minimizing the ARPS between prefix plan and suffix plan. The (\ref{eq: Complete_policy}) can be solved via any LP solvers i.e., Gurobi \cite{gurobi} and CPLEX \footnote{\url{https://www.ibm.com/analytics/cplex-optimizer}}. Once the optimal solutions $y_{x,u}^{*}$ and $z_{x,u}^{*}$ are generated, we can synthesize the optimal policies $\boldsymbol{\pi}^{*}_{pre}$ and $\boldsymbol{\pi}^{*}_{suf}$ via (\ref{eq:prefix_policy}) and (\ref{eq:suffix_policy}). The complete optimal  policy $\boldsymbol{\pi}^{*}$ can be obtained by concatenating $\pi^{*}_{pre}$ and $\pi^{*}_{suf}$ for all states of $\mathcal{R}$.

Since $\boldsymbol{\pi}^{*}$ is defined over $\mathcal{R}$, to execute the optimal
policy over $\ensuremath{\mathcal{M}}$, we still need to map $\boldsymbol{\pi}^{*}$
to an optimal finite-memory policy $\boldsymbol{\boldsymbol{\boldsymbol{\boldsymbol{\mu}}}}^{*}$ of $\ensuremath{\mathcal{M}}$.
Suppose the agent starts from an initial state $x_{0}=\left(s_{0},l_{0},q_{0}\right)$
and the distribution of optimal actions at $t=0$ is given by $\boldsymbol{\pi}^{*}\left(x_{0}\right)$.
Taking an action $u_{0}^{\mathcal{R}}$ according to $\boldsymbol{\pi}^{*}\left(s_{0}\right)$,
the agent moves to $s_{1}$ and observes its current label $l_{1}$,
resulting in $x_{1}=\left(s_{1},l_{1},q_{1}\right)$ with $q_{1}=\delta\left(q_{0},u_{0}\bigr|_{\ensuremath{\mathcal{A}}}^{\mathcal{R}}\right)$.
Note that $q_{1}$ is deterministic if $u_{0}\bigr|_{\ensuremath{\mathcal{A}}}^{\mathcal{R}}\neq\epsilon$.
The distribution of optimal actions at $t=1$ now becomes $\boldsymbol{\pi}^{*}\left(x_{1}\right)$.
Repeating this process infinitely will generate a path $\boldsymbol{x}_{\mathcal{R}}^{\boldsymbol{\pi}^{*}}=x_{0}x_{1\ldots}$
over $\mathcal{R}$, corresponding to a path $\boldsymbol{s}=s_{0}s_{1}\ldots$
over $\ensuremath{\mathcal{M}}$ with associated labels $L(\boldsymbol{s})=l_{0}l_{1}\ldots$.
Such a process is presented in Algorithm \ref{Alg1}. Since the state
$x_{t}$ is unique given the agent's past path $\boldsymbol{s}\left[:t\right]$
and past labels $L(\boldsymbol{s}\left[:t\right])$ up to $t$, the
optimal finite-memory policy is designed as
\begin{equation}
\boldsymbol{\mu}^{*}\left(\boldsymbol{s}\left[:t\right],L(\boldsymbol{s}\left[:t\right])\right)=\left\{ \begin{array}{cc}
\boldsymbol{\pi}^{*}\left(x_{t}\right), & \text{for }u_{t}\bigr|_{\ensuremath{\mathcal{M}}}^{\mathcal{R}}=a,\\
\text{  } & \text{  }\\
0, & \text{for }\ensuremath{u_{t}\bigr|_{\ensuremath{\mathcal{A}}}^{\mathcal{\mathcal{R}}}=\epsilon}.\\
\end{array}\right.\label{eq:policy mapping}
\end{equation}
From definition \ref{def:relaxed-product}, the state $s_{t}$
in $x_{t}$ remains the same if $\ensuremath{u_{t}\bigr|_{\ensuremath{\mathcal{\mathcal{M}}}}^{\mathcal{\mathcal{R}}}=\epsilon}$ which gives rise to $\boldsymbol{\mu}^{*}\left(\boldsymbol{s}\left[:t\right],L(\boldsymbol{s}\left[:t\right])\right)=0$
in (\ref{eq:policy mapping}). 

\begin{thm}
Given a PL-MDP and an LTL formula $\phi$, the optimal policy $\boldsymbol{\mu}^{*}$ from (\ref{eq: Complete_policy}) and in (\ref{eq:policy mapping}) solves the Problem \ref{Prob1} exactly s.t. achieve multiple objectives in order of decreasing priority:
	1) if $\phi$ is fully feasible, $\mathbb{\Pr}_{\mathcal{M}}^{\boldsymbol{\boldsymbol{\boldsymbol{\boldsymbol{\boldsymbol{\mu}}}}}}\left(\phi\right)\geq\gamma$
	with $\gamma\in (0,1]$; 2) if $\phi$ is infeasible, satisfy $\phi$ as much as possible via minimizing AVPS; 3) minimize AEPS over the infinite horizons. 
\end{thm}

\begin{proof}
First, the optimal policy $\boldsymbol{\pi}^{*}$ solves Problem \ref{problem2} exactly. Such a conclusion can be verified directly based on Theorem \ref{Thm_Properties}, Lemma \ref{lemma:prefix}, and Lemma \ref{lemma:suffix}. Because Problem \ref{Prob1} and Problem \ref{problem2} are equivalent, the policy projection in (\ref{eq:policy mapping}) finds a policy in $\mathcal{M}$ that solves Problem \ref{Prob1} exactly \cite{Baier2008}.
\end{proof}

In Alg. \ref{Alg1}, the overall policy synthesis is summarized in lines 1-12 of Alg. Note that the optimization process of suffix plan (line 9-11) is applied to every AMEC of $\mathcal{R}$.  After obtaining the complete optimal policies, the process of executing such a policy for PL-MDP $\mathcal{M}$ is outlined in lines 13-24.

\begin{rem}
The complete policy developed in the work can handle both feasible and infeasible cases simultaneously, and AMECs of relaxed product MDP are computed off-line once based on the algorithms of \cite{Baier2008}.
\end{rem}

\subsection{Complexity Analysis}
\label{sec:complexity}
The maximum number of states is $\left|X\right|=\left|S\right|\times\left|L_{max}\left(S\right)\right|\times\left|Q\right|$,
where $\left|Q\right|$ is determined by the LDBA $\ensuremath{\mathcal{A}}_{\phi}$,
$\left|S\right|$ is the size of the environment, and $L_{max}\left(S\right)$
is the maximum number of labels associated with a state $s\in S$.
Due to the consideration of relaxed product MDP and the extended actions,
the maximum complexity of actions available at $x_{0}=\left(s_{0},l_{0},q_{0}\right)\in X$
is $O\left(\left|\mathcal{A}\left(s\right)\right|\times\left|\Sigma\cup\left\{ \epsilon\right\}\right|\right)$.
From \cite{Baier2008}, the complexity of computing AMECs for $\mathcal{R}$
is $O\left(\left|X\right|^{2}\right)$. The size of LPs
in (\ref{eq: Pre_LP}) and (\ref{eq:suffix_LP}) is linear with respect
to the number of transitions in $\mathcal{R}$ and can be solved in
polynomial time \cite{Dantzig1998}.

\section{Case Studies\label{sec:Case}}

Here considers a mobile agent operating in a grid environment, which is a commonly used benchmark for probabilistic model checking in the literature \cite{Sadigh2014,Guo2018,hasanbeig2018logically, hasanbeig2019certified}.
There are
properties of interest associated with the cells. To model environment
uncertainties, these properties are assumed to be probabilistic. We consider the same motion uncertainties as Example \ref{example1}. The
agent is allowed to transit between adjacent cells or stay in a cell,
i.e., the action space is $\left\{ \mathtt{Up,Right,Down,Left},\mathtt{Stay}\right\}$, and the action costs are $[ 3, 4, 2, 3, 1]$.
To model the agent\textquoteright s motion uncertainty caused by actuation
noise and drifting, the agent's motion is also assumed to be probabilistic. For instance, the robot may successfully take the desired action with a
probability of $0.85$, and there's a probability of $0.15$ to take other perpendicular actions based on uniform distributions. There is no motion uncertainty for the action of "$\text{Stay}$".
In the following cases, the algorithms developed
in Section \ref{sec:Solution} are implemented, where $\beta=100$
is employed to encourage a small violation of the desired task if the task
is infeasible. The desired satisfaction probability is set
as $\gamma=0.9$. Gurobi \cite{gurobi} is used to solve the linear program problems
in (\ref{eq: Pre_LP}) and (\ref{eq:suffix_LP}). All algorithms are
implemented in Python 2.7, and Owl \cite{Kretinsky2018} is used to
convert LTL formulas into LDBA. All simulations are carried out on
a laptop with a 2.60 GHz quad-core CPU and 8GB of RAM. 

\subsection{Case 1: Feasible Tasks}
\label{subsec:case1}

\begin{figure}
	\centering{}\includegraphics[scale=0.26]{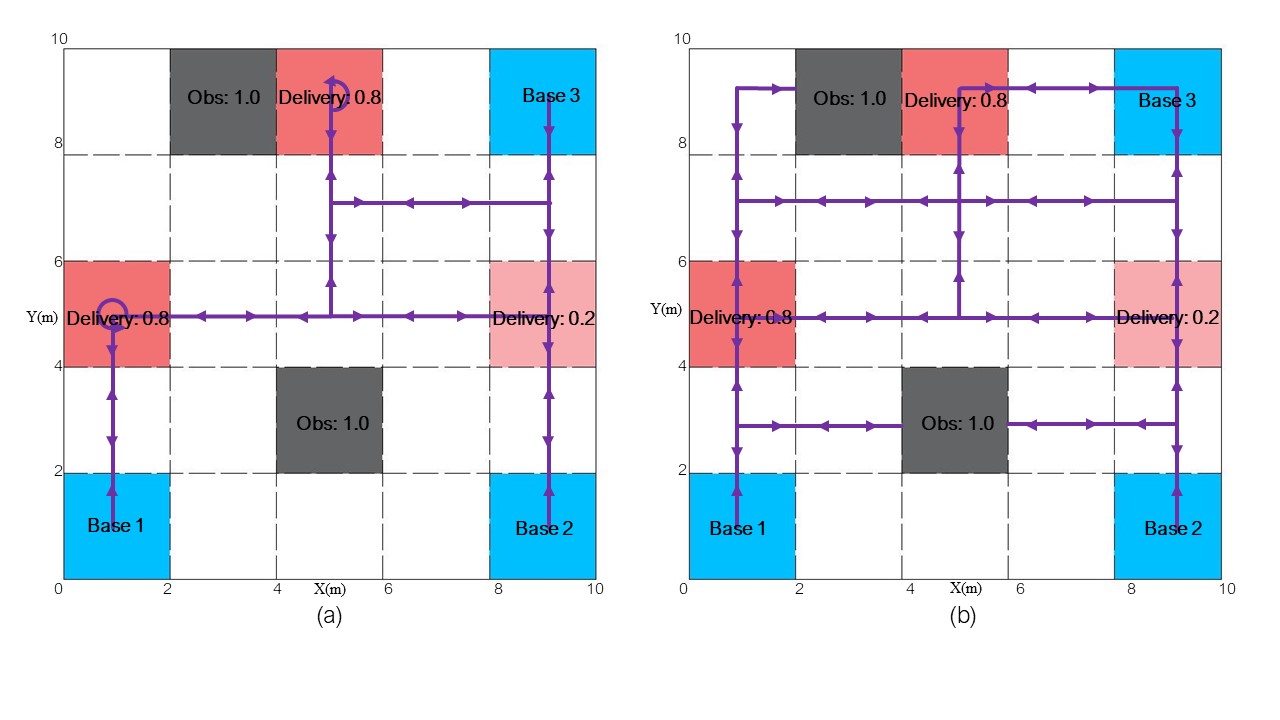}\caption{\label{fig:case_study1.1}Simulated trajectories
		by the optimal policy in (a) and the Round-Robin policy in (b). The
			line arrows represent the directions of movement and the circle arrows
			represent the $\mathtt{Stay}$ action.}
\end{figure}
\begin{figure}
	\centering{}\includegraphics[scale=0.45]{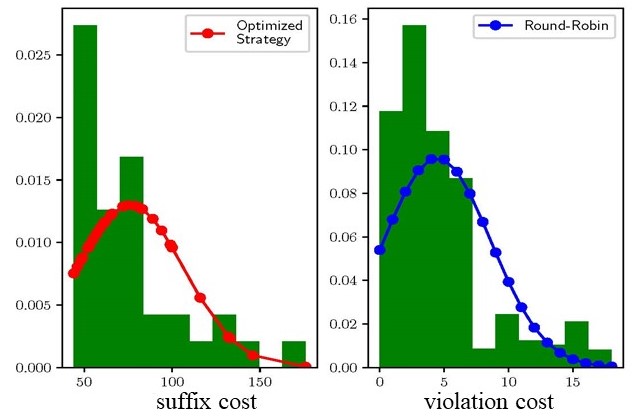}\caption{\label{fig:case_study1.2} (a) Normalized distribution of the plan
		suffix cost under the optimal policy. (b) Normalized distribution
		of the violation cost under the Round-Robin policy.}
\end{figure}

This case considers motion planning in an environment where the desired
task can be completely fulfilled.
Suppose the agent is required to perform a surveillance task in a
workspace as shown in Fig. \ref{fig:case_study1.1} and the task specification
is expressed in the form of LTL formula as
\begin{equation}
\begin{aligned}\varphi_{case1}= & \left(\oblong\lozenge\mathtt{base}1\right)\land\left(\oblong\lozenge\mathtt{base}2\right)\land\left(\oblong\lozenge\mathtt{base}3\right)\\
& \land\oblong\left(\varphi_{one}\rightarrow\Circle\left(\left(\lnot\varphi_{one}\right)\cup\mathtt{Delivery}\right)\right)\\
& \land\oblong\lnot\mathtt{Obs},
\end{aligned}
\label{eq:case1}
\end{equation}
where $\varphi_{one}=\mathtt{base}1\lor\mathtt{base}2\lor\mathtt{base}3$.
The LTL formula in (\ref{eq:case1}) means that the agent visits one
of the base stations and then goes to one of the delivery stations
while avoiding obstacles. All base stations need to be visited. Based on the environment and motion uncertainties, the LTL formula 
$\varphi_{case1}$ with respect to PL-MDP is feasible.
The
corresponding LDBA has $35$ states and $104$ transitions, and the
PL-MDP has 28 states. It took $11.2$s to construct the relaxed product
MDP and $0.15$s to synthesize the optimal policy via Alg. \ref{Alg1}.
To demonstrate the efficiency, we also compare the optimal policies generated from this
with the Round-Robin policy.

\begin{figure}
	\centering{}\includegraphics[scale=0.24]{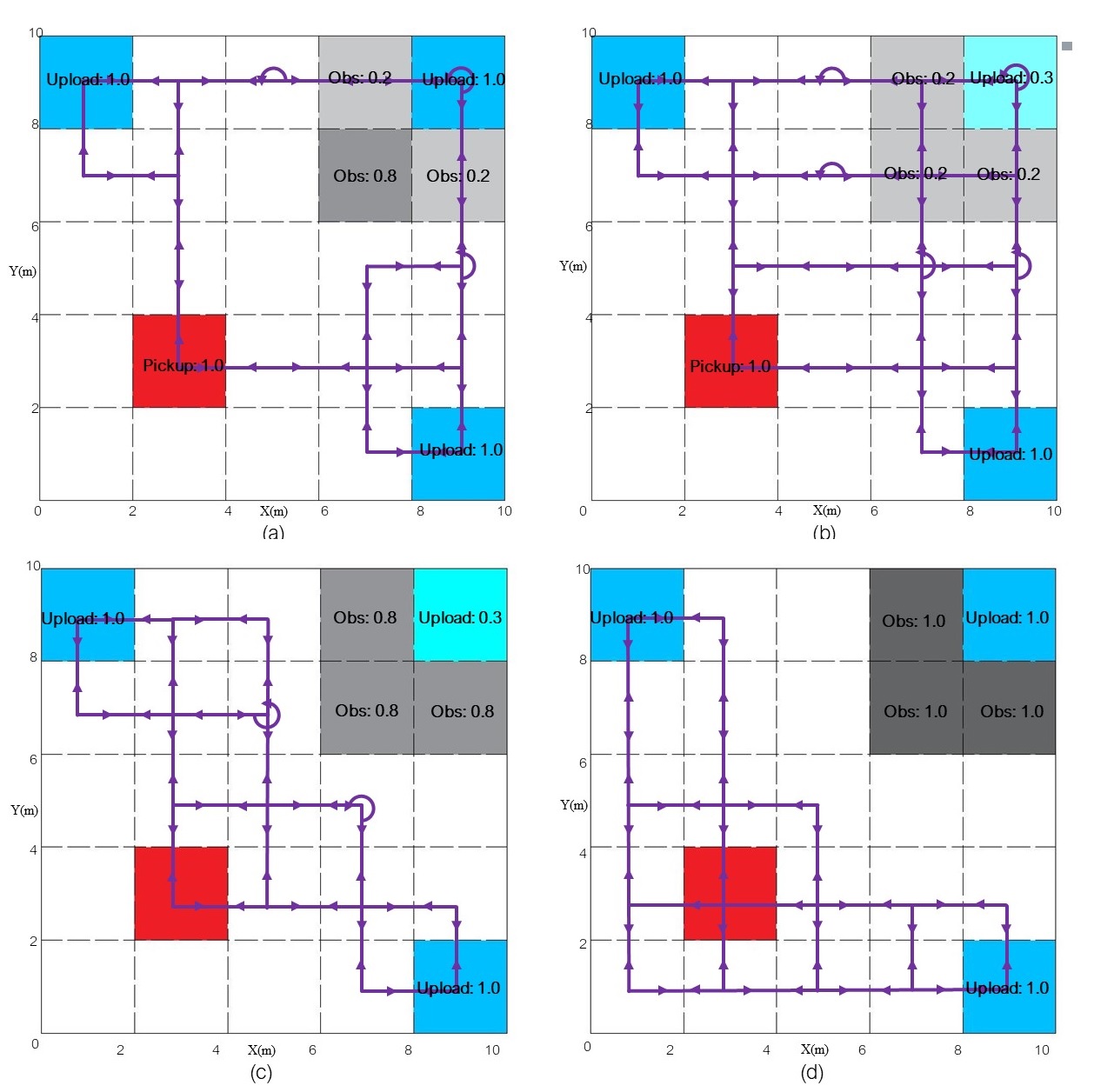}\caption{\label{fig:case_study2.1} Simulated trajectories by the optimal policy
		for different environments.}
\end{figure}

Fig. \ref{fig:case_study1.1} (a) and (b) show the
	trajectories generated by our optimal policy and the Round-Robin policy,
	respectively. The arrows represent the directions of movement, and
	the circles represent the $\mathtt{Stay}$ action. Clearly, the optimal
	policy is more efficient in the sense that fewer cells were visited
	during mission operation. In Fig. \ref{fig:case_study1.2}, 1000 Monte
	Carlo simulations were conducted. Fig. \ref{fig:case_study1.2} (a)
	shows the distribution of the plan suffix cost. It indicates that,
	since the task is completely feasible, the optimal policy in this
	work can always find feasible plans with zero AVPS. Since Round-Robin
	policy would select all available actions enabled at each state of
	AMEC, Fig. \ref{fig:case_study1.2} (b) shows the distribution of
	the violation cost under Round-Robin policy. 

\subsection{Case 2: Infeasible Tasks }

\begin{figure}
	\centering{}\includegraphics[scale=0.45]{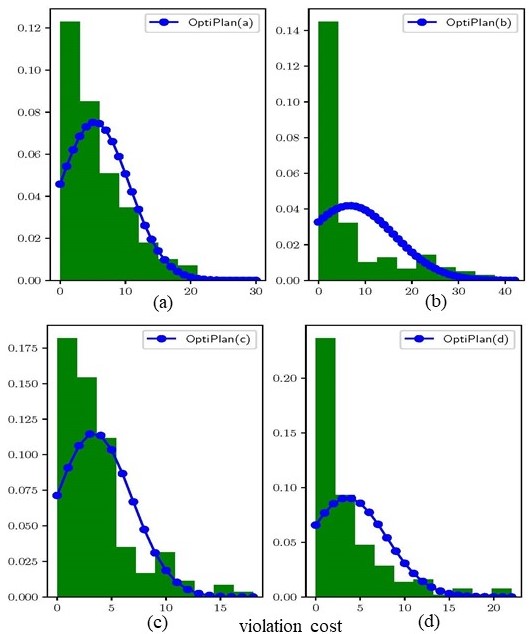}\caption{\label{fig:case_study2.2} Normalized distributions of the violation
		cost for different environments.}
\end{figure}

This case considers motion planning in an environment where the desired
task might not be fully executed. In Fig. \ref{fig:case_study2.1},
Suppose the agent is tasked to visit the $\mathtt{pickup}$ station
and then goes to one of the $\mathtt{upload}$ stations while avoiding
obstacles. In addition, the agent is not allowed to visit the $\mathtt{pickup}$
station before visiting an $\mathtt{upload}$ station, and all $\mathtt{upload}$
stations need to be visited. Such a task can be written in an LTL
formula as
\begin{equation}
\begin{aligned}\varphi_{case2}= & \oblong\lozenge\mathtt{Pickup}\land\oblong\lnot\mathtt{Obs}\\
& \land\oblong\left(\mathtt{Pickup}\rightarrow\Circle\left(\left(\lnot\mathtt{Pickup}\right)\cup\varphi_{one}\right)\right)\\
& \land\oblong\lozenge\mathtt{Upload1}\land\oblong\lozenge\mathtt{Upload2}\land\oblong\lozenge\mathtt{Upload3},
\end{aligned}
\label{eq:case2}
\end{equation}
where $\varphi_{one}=\mathtt{Upload1}\lor\mathtt{Upload2}\lor\mathtt{Upload3}$.
Fig. \ref{fig:case_study2.1} (a)-(c) show infeasible tasks
since the cells surrounding $\mathtt{Upload2}$ are occupied by obstacles
probabilistically. Fig. \ref{fig:case_study2.1} (d) shows an infeasible
environment since $\mathtt{Upload2}$ is surrounded by obstacles for
sure and can never be reached. 

Simulation results show how the relaxed product MDP can synthesize
an optimal plan when no AMECs or no ASCCs exist. Note that the\textcolor{blue}{{}
}algorithm in \cite{Guo2018} returns no solution if no ASCCs exist.
The resulting LDBA has $43$ states and $136$ transitions, and it
took 0.15s on average to synthesize the optimal policy. The
	simulated trajectories are shown in Fig. \ref{fig:case_study2.1}
	with arrows indicating the directions of movement. In Fig. \ref{fig:case_study2.1}
	(a), since the probability of $\mathtt{Upload2}$ is high and the
	probabilities of surrounding obstacles are relatively low, the planning
	tries to complete the desired task $\varphi_{case2}$. In Fig. \ref{fig:case_study2.1}
	(b) and (c), the probability of $\mathtt{Upload2}$ is 0.3 while the
	probabilities of surrounding obstacles are 0.2 and 0.8, respectively.
	The agent still tries to complete $\varphi_{case2}$ by visiting $\mathtt{Upload2}$
	in Fig. \ref{fig:case_study2.1} (b) while the agent is relaxed to
	not visit $\mathtt{Upload2}$ in Fig. \ref{fig:case_study2.1} (c)
	due to the high risk of running into obstacles and low probability
	of $\mathtt{Upload2}$. Since $\varphi_{case2}$ is completely infeasible
	in Fig. \ref{fig:case_study2.1} (d), the motion plan is revised to
	not visit $\mathtt{Upload2}$ and select paths with the minimum violation
	and implementation cost to mostly satisfy $\varphi_{case2}$. To illustrate the ability to minimize AVPS for infeasible cases, we analyze the violations such that Fig. 9 shows the distribution of AVPS for 1000 Monte Carlo simulations corresponding to the four different infeasible cases in
	Fig. \ref{fig:case_study2.1}, respectively. It can be observed that there is a high probability of obtaining a small AVPS with this framework.

\subsection{Parameter Analysis and Results Comparison}

\begin{table}
	\caption{\label{tab:parameter}The expected execution costs using different parameters $\eta$ for $\phi_{case1}$.}
	\centering{}%
	\begin{tabular}{c|cccc}
		\hline 
		$\eta$ & Total cost & Cyclic cost & Mean Cost \tabularnewline
		parameter & Prefix & Suffix & Suffix \tabularnewline
		\hline 
		$0$ & 36.4 & 178.5 & 2.823 \tabularnewline
		$0.2$ & 36.7  & 66.1 & 2.545 \tabularnewline
		$0.4$ & 36.7 & 65.8 & 2.540 \tabularnewline
		$0.6$ & 39.4 & 62.1 & 2.538 \tabularnewline
		$0.8$ & 50.9 & 57.2 & 2.520 \tabularnewline
		$1.0$ & 115.6 & 55.9 & 2.512 \tabularnewline
		\hline 
	\end{tabular}
\end{table}

\begin{figure}[t]
	\centering{}\includegraphics[scale=0.31]{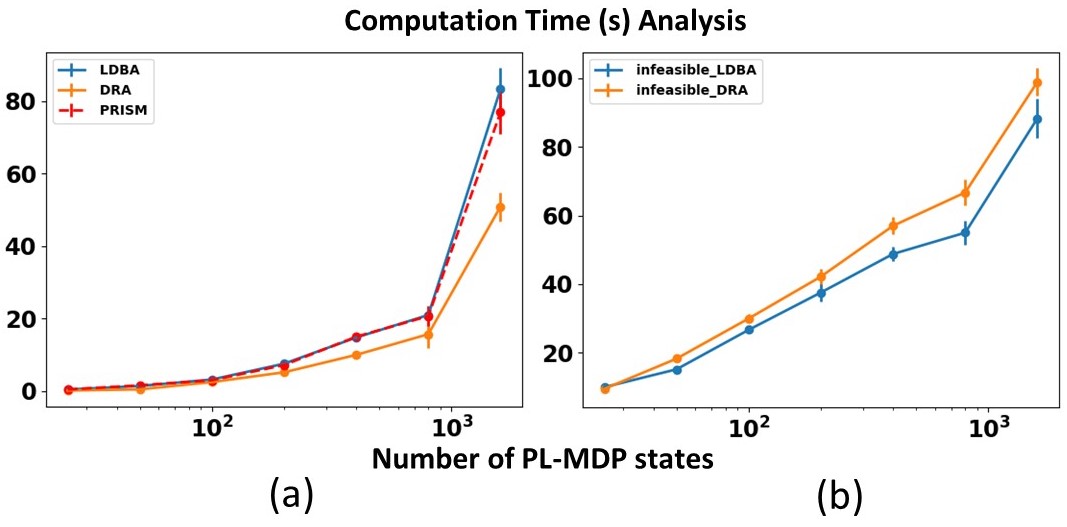}\caption{\label{fig:complexity}Computation time for different methods. (a) Computation Time of solving the optimization process using "LDBA", "DRA", PRISM for feasible cases. (b)
			Computation Time of overall process (construct models and solve the optimization) using "LDBA", "DRA" for infeasible cases.}
\end{figure}

In this section, we first apply the feasible task $\phi_{case1}$ to analyze the effect of $\eta$ in (\ref{eq: Complete_policy}) on the trade-off between the optimal expected execution costs of prefix and suffix plans. The results are shown in Table \ref{tab:parameter}.
Then we compared our framework referred as "LDBA" with widely used model-checking tool PRISM\cite{Kwiatkowska2011}. To implement PRSIM with the PL-MDP, we use the package \cite{Guo2018} to translate the relaxed product automaton into PRISM language and verify the LDBA accepting condition. We select the option of PRISM  "multi-objective property" that finds the policies satisfying task with the risk lower-bounded by $1-\gamma$ while minimizing cumulative reward. The tool PRSIM can only synthesize the optimal prefix policy and does not support the optimization of suffix structure to handle infeasible cases.  Thus, we compare the computation time for the feasible LTL task $\phi_{case1}$ with its environments, divide each grid of the environment to construct various workspace sizes, and run $10$ times for each environment with the same size where the initial locations are selected from uniform distributions. We compare such complexity of the feasible case $\phi_{case1}$ with the works \cite{Guo2018},\cite{Forejt2011,Forejt2012} named as "DRA" that uses the Deterministic Rabin Automaton and standard product MDP to synthesize solutions. The results are shown in Fig. \ref{fig:complexity} (a). we can see the computation time of the optimization process with PRISM is almost the same. And since relaxed product MDP is more connected than standard product MDP, the computation time of our work is a little higher than the works  \cite{Guo2018},\cite{Forejt2011,Forejt2012} for feasible cases.

As for infeasible cases, even though the algorithm \cite{Guo2018} returns no solutions for cases described in Fig. \ref{example1} (a), Fig. \ref{example_suffix}, and Fig. \ref{fig:case_study2.1} (d), it still has solutions for case  $\phi_{case2}$ in  Fig. \ref{fig:case_study2.1} (a) (b) (c). We refer to the algorithm \cite{Guo2018} as "infeasible DRA", and our framework as "infeasible LDBA". To analyze the computational time, we select the environment of Fig. \ref{fig:case_study2.1} (b). We also divide each grid to generate various workspace sizes and run $10$ times for each environment with the same size where the initial locations are selected from uniform distributions. The results of computation time for the overall process are shown in Fig. \ref{fig:complexity} (b). It shows that our algorithm has better computational performance because the work  \cite{Guo2018} needs to construct AMECs first to check the feasibility and then construct ASCCs to start the optimization process, which is computationally expensive.

\subsection{Case 3: Large Scale Analysis}

\begin{table}
	\caption{\label{tab:case3_1}The comparison of workspace size and computation
		time.}
	\centering{}%
	\begin{tabular}{c|cccc}
		\hline 
		Workspace & $\mathcal{M}$ & $\mathcal{R}$ & AMECs & $\boldsymbol{\pi}^{*}$\tabularnewline
		size{[}cell{]} & Time{[}s{]} & Time{[}s{]} & Time{[}s{]} & Time{[}s{]}\tabularnewline
		\hline 
		$5\times5$ & 0.14 & 0.56 & 0.64 & 0.45\tabularnewline
		$10\times10$ & 1.59 & 1.34 & 1.88 & 3.20\tabularnewline
		$30\times30$ & 25.4 & 5.20 & 7.41 & 20.71\tabularnewline
		$50\times50$ & 460.1 & 28.95 & 25.89 & 124.03\tabularnewline
		$100\times100$ & 843.9 & 41.47 & 39.80 & 276.05\tabularnewline
		\hline 
	\end{tabular}
\end{table}
This case considers motion planning in a larger scale problem. To
show the efficiency of using LDBA, we first repeat
the task of Case 1 for different workspace sizes. Table \ref{tab:case3_1}
lists the computation time for the construction of PL-MDP, the relaxed
product MDP, AMECs, and the optimal plan $\boldsymbol{\pi}^{*}$ in different workspace
sizes. 

\begin{figure}[t]
	\centering{}\includegraphics[scale=0.25]{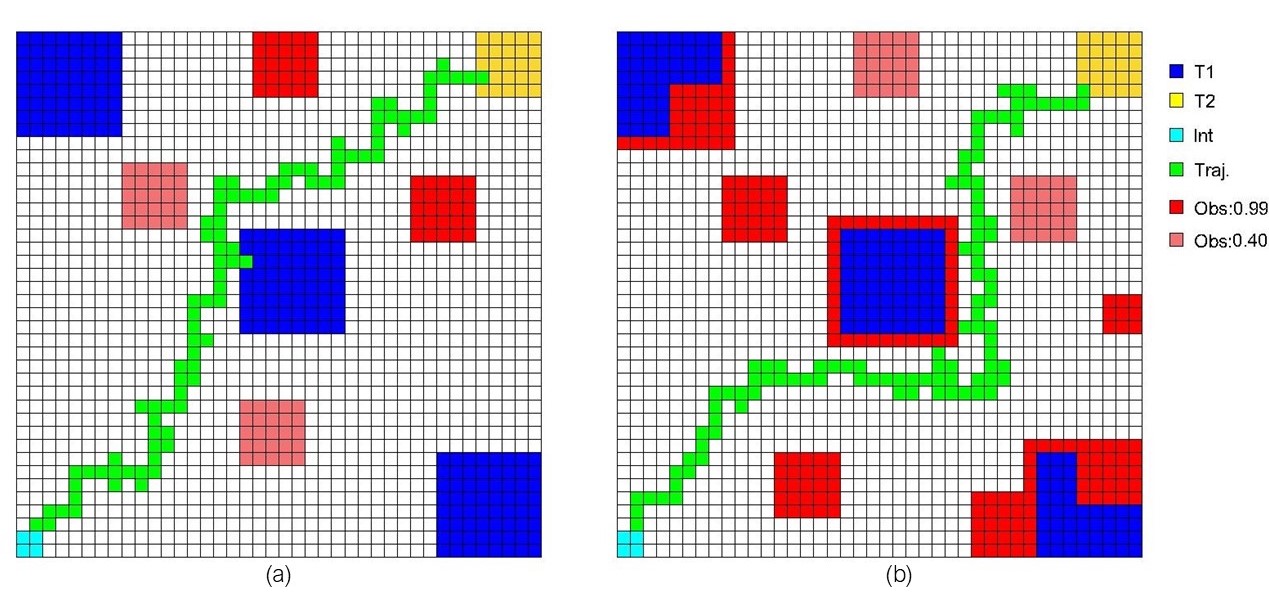}\caption{\label{fig:large_scale}Simulation results for the
			specification $\varphi_{case3}$. (a) shows a feasible case and (b)
			shows an infeasible case where $T1$ can not be visited.}
\end{figure}

To demonstrate the scalability and computational complexity, consider
a $40\times40$ workspace as in Fig. \ref{fig:large_scale}. The desired
task expressed in an LTL formula is given by
\[
\varphi_{case3}=\boxempty\lnot\mathtt{Obs}\wedge\diamondsuit\mathtt{T1}\wedge\boxempty\left(\mathtt{T1}\rightarrow\varbigcirc\left(\lnot\mathtt{T1}\cup\mathtt{T2}\right)\right),
\]
where $\mathtt{T1}$ and $\mathtt{T2}$ represent two targets properties
to be visited sequentially. The agent starts from the left corner
(i.e., the light blue cell). The LDBA associated with $\varphi_{case3}$
has $6$ states and $17$ transitions, and it took $27.3$ seconds
to generate an optimal plan. The simulation trajectory is shown in
Fig. \ref{fig:large_scale}. Note that AMECs of $\mathcal{P}$ only
exist in Fig. \ref{fig:large_scale} (a). Neither AMECs nor ASCCs
of $\mathcal{P}$ exist in Fig. \ref{fig:large_scale} (b), since
$\mathtt{T1}$ is surrounded by obstacles and can not be reached.
Clearly, the desired task $\varphi_{case3}$ can
	be mostly and efficiently executed whenever the task is feasible or
	not.

\subsection{Mock-up Office Scenario}

In this section, we verify our algorithm for high-level decision-making problems in a real-world environment, and show that the framework can be adopted with any stochastic abstractions and low-level noisy controllers to formulate a hierarchical architecture.
Consider a TIAGo robot operating in an office environment
	as shown in Fig. \ref{fig:case_study3}, which can be modeled in ROS
	Gazebo in Fig. \ref{fig:case_study3}. The mock-up office consists
	of $6$ rooms $S_{i}$, $i=0,\ldots,5$, and $4$ corridors $C_{i}$,
	$i=0,\ldots,3$. The TIAGo robot can follow a collision-free path
	from the center of one region to another without crossing other regions
	using obstacle-avoidance navigation. The marked area $C_{2}$ represents
	an inclined surface, where more control effort is required by TIAGo
	robot to walk through. To model motion uncertainties, it is assumed
	that the robot can follow its navigation controller moving
	to the desired region with a probability of $0.9$ and fail by moving to
	the adjacent region with a probability of $0.1$. The resulting MDP has
	$10$ states. 

The LTL task is formulated as 
	\[
	\varphi_{case4}=\oblong\lozenge S_{0}\land\oblong\lozenge S_{1}\land\oblong\lozenge S_{2}\land\oblong\lozenge S_{3}\land\oblong\lozenge S_{4}\land\oblong\lozenge S_{5},
	\]
	which requires the robot to periodically serve all rooms. Its corresponding
	LDBA has $6$ states with $6$ accepting states and the relaxed product
	MDP has $60$ states. The simulated trajectories are shown in Fig.
	\ref{fig:case_study3} (a) and (b). The task is satisfied exactly
	in Fig. \ref{fig:case_study3} (a) because all rooms are accessible.
	In Fig. \ref{fig:case_study3} (b), $\varphi_{case4}$ is only
	feasible since rooms $S_{2}$ and $S_{5}$ are closed. Hence, the
	robot revises its plan to only visit rooms $S_{0}$, $S_{1}$, $S_{3}$,
	and $S_{4}$. In both Fig. \ref{fig:case_study3} (a) and (b), $C_{2}$
	is avoided for energy efficiency.


\begin{figure}
	\centering{}\includegraphics[scale=0.40]{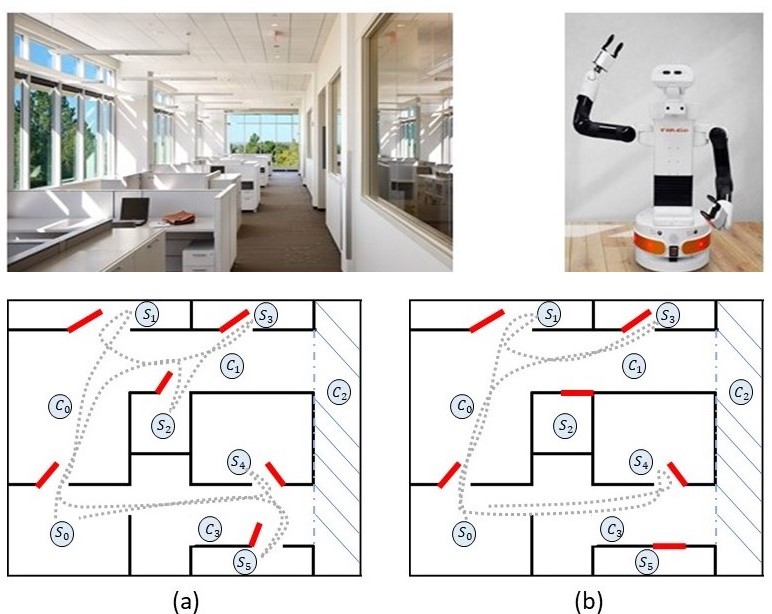}\caption{\label{fig:case_study3}  Generated trajectories for the task using  $\varphi_{case4}$ in the mock-up office scenario
		with the TIAGo robot.
	}
\end{figure}

\section{Conclusion}

A plan synthesis algorithm for probabilistic motion
	planning is developed for both feasible and infeasible tasks.
	LDBA is employed to evaluate the LTL satisfaction with the B\"uchi acceptance condition. The extended actions and relaxed product MDP are developed to
	allow probabilistic motion revision if the workspace is not fully feasible to the
	desired mission. Cost optimization is studied in both plan
	prefix and plan suffix of the trajectory.
	Inspired by the existing works, e.g., \cite{cai2021safe},
	future research will consider the optimization of multi-objective over
	continuous space with safety-critical constraints. Additional in-depth research includes extending this work to multi-agent systems with cooperative
	tasks.
	
\section{Acknowledgement}
We thank Meng Guo for the open-source software.
We also thank the editor and anonymous reviewers for their time and efforts in helping improve the paper.

\bibliographystyle{IEEEtran}
\bibliography{BibMaster}

\end{document}